\documentclass[11pt,a4paper]{article}
\usepackage{cmap}
\usepackage[utf8x]{inputenc}
\usepackage[T1]{fontenc}
\usepackage[english]{babel}
\usepackage{amssymb,amsmath}
\usepackage{amsthm}
\usepackage{bbm}
\usepackage{microtype}
\usepackage{lmodern}

{
\rmfamily
\DeclareFontShape{T1}{lmr}{m}{sc}{<->ssub*cmr/m/sc}{}
\DeclareFontShape{T1}{lmr}{b}{sc}{<->ssub*cmr/b/sc}{}
\DeclareFontShape{T1}{lmr}{bx}{sc}{<->ssub*cmr/bx/sc}{}
}

\newcommand{\thmheadercommand}[1]{\textbf{\scshape{}#1.\\*}}

\newtheoremstyle{yannthm}{\topsep}{\topsep}{\slshape}{}{\scshape\bfseries}{.}{.5em}{%
\thmname{#1}\thmnumber{ #2}\thmnote{#3}%
}
\newtheoremstyle{yannthm2}{\topsep}{\topsep}{}{}{\scshape\bfseries}{.}{.5em}{%
\thmname{#1}\thmnumber{ #2}\thmnote{#3}%
}

\def\d{\operatorname{d}\!{}}

\def\R{{\mathbb{R}}}

\newcommand{\deq}{\mathrel{\mathop:}=}

\newcommand{\from}{\colon} 

\def\eps{\varepsilon}
\renewcommand{\epsilon}{\varepsilon}
\renewcommand{\phi}{\varphi}

\DeclareMathOperator{\Hess}{Hess}

\DeclareMathOperator{\Var}{Var}

\let\oldPr\Pr
\renewcommand{\Pr}{\oldPr\nolimits}
\newcommand{\E}{\mathbb{E}}

\DeclareMathOperator{\Id}{Id}

\newcommand{\abs}[1]{\left|\mskip1mu#1\right|}
\newcommand{\norm}[1]{\left\|#1\right\|}
\newcommand{\scal}[2]{\left< \, #1 \mid #2 \, \right>}
\newcommand{\1}{\mathbbm{1}}

\DeclareMathOperator*{\argmin}{arg\,min}

{
\theoremstyle{yannthm}
\newtheorem{defi}{Definition}
\newtheorem*{defi*}{Definition}
\newtheorem{prop}[defi]{Proposition}
\newtheorem*{prop*}{Proposition}

\newtheorem*{thm*}{Theorem}

\newtheorem*{lem*}{Lemma}

\newtheorem*{cor*}{Corollary}
\newtheorem{ex}[defi]{Example}
\newtheorem*{ex*}{Example}

\newtheorem*{subenonce}{}

\theoremstyle{yannthm2}

\newtheorem*{exo*}{Exercise}

\newtheorem{rem}[defi]{Remark}
\newtheorem*{rem*}{Remark}

\newtheorem*{subenonce2}{}

}

\newenvironment{enonce2}[1]{\begin{subenonce2}[#1]}{\end{subenonce2}}

\newcommand{\transp}[1]{#1^{\!\top}\!}

\usepackage[pdfborder={0 0 0}]{hyperref}
\usepackage{url}
\usepackage{graphicx}
\usepackage{multirow}

\newcommand{\e}{\mathrm{e}}
\renewcommand{\d}{\ensuremath{\hspace{0.05em}\delta\hspace{-.05em}}}
\newcommand{\grad}{\nabla\!}
\renewcommand{\L}{\mathcal{L}}
\newcommand{\Lout}{\L_\mathrm{out}}
\newcommand{\Lin}{\L_\mathrm{in}}
\newcommand{\A}{\mathcal{A}}

\newcommand{\actf}{s}
\newcommand{\deractf}{r}

\DeclareMathOperator{\sigm}{sigm}
\DeclareMathOperator{\diag}{diag}
\newcommand{\ssum}{{\textstyle\sum}}
\newcommand{\nat}{_{\mathrm{nat}}}
\newcommand{\datnat}[1]{_{\mathrm{nat},#1}}
\newcommand{\natnorm}[1]{\norm{#1}\nat}
\newcommand{\datnatnorm}[2]{\norm{#1}\datnat{#2}}
\newcommand{\unat}{_{\mathrm{u\text{-}nat}}}
\newcommand{\unatnorm}[1]{\norm{#1}\unat}
\newcommand{\bp}{_{\mathrm{bp}}}
\newcommand{\datbp}[1]{_{\mathrm{bp},#1}}
\newcommand{\bpnorm}[1]{\norm{#1}\bp}
\newcommand{\datbpnorm}[2]{\norm{#1}\datbp{#2}}
\newcommand{\fmnorm}[1]{\norm{#1}_{\mathrm{F\text{-}mod}}}

\newcommand{\qdnorm}[1]{\norm{#1}_{\mathrm{qd}}}

\newcommand{\opnorm}[1]{\norm{#1}_{\mathrm{op}}}
\newcommand{\uopnorm}[1]{\norm{#1}_{\mathrm{u\text{-}op}}}

\newcommand{\out}{_\mathrm{out}}
\newcommand{\D}{\mathcal{D}}

\newcommand{\dw}{\ensuremath{\hspace{0.05em}\d\hspace{-.06em}w\hspace{0.05em}}}

\title{Riemannian metrics for neural networks I: Feedforward networks}
\author{Yann Ollivier}

\begin{document}

\maketitle

\begin{abstract}
We describe four algorithms for neural network training, each adapted to different scalability
constraints. These algorithms are mathematically principled and invariant under
a number of transformations in data and network representation, from which
performance is thus independent. These
algorithms are obtained from the setting of differential geometry, and are
based on either the natural gradient using the Fisher information matrix, or on Hessian methods, scaled down in a
specific way to allow for scalability while keeping some of their key mathematical
properties.
\end{abstract}

The most standard way to train neural networks, backpropagation, has
several known shortcomings. Convergence can be quite slow.
Backpropagation is sensitive to data representation: for instance, even
such a simple operation as exchanging $0$'s and $1$'s on the input layer
will affect performance (Figure~\ref{fig:bptraj}), because this amounts
to changing the parameters (weights and biases) in a non-trivial way,
resulting in different gradient directions in parameter space, and better
performance with $1$'s than with $0$'s. (In the related context of
restricted Boltzmann machines, the standard training technique by
gradient ascent favors setting hidden units to $1$, for similar reasons
\cite[Section~5]{IGO}.) This specific phenomenon disappears if, instead
of the logistic function, the hyperbolic tangent is used as the
activation function, or if the input is normalized. But this will not
help if, for instance, the activities of internal units in a multilayer
network are not centered on average. Scaling also has an effect on
performance: for instance, a common recommendation \cite{LBOM96} is to
use $1.7159\tanh(2x/3)$ instead of just $\tanh(x)$ as the activation
function.

\begin{figure}
\begin{center}
\includegraphics[width=.7\columnwidth]{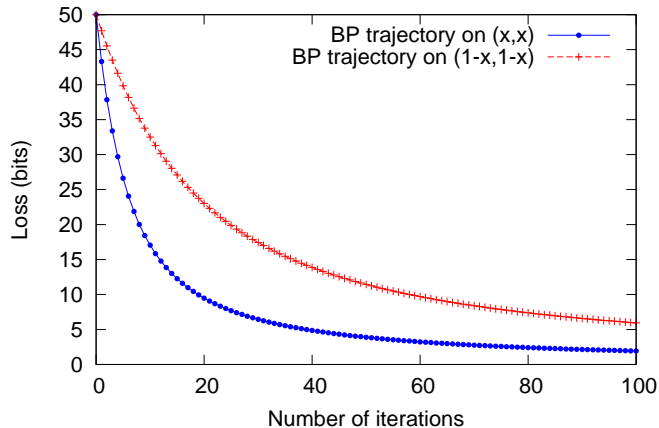}

\caption{\label{fig:bptraj}Backpropagation learns faster with 1's. A
neural network with two layers of size $50$ (no hidden layer) is trained to reproduce
its input. A random binary sequence $x$ of length 50 with
$75\%$ of 1's is generated. The network is trained on the input-output pair
$(x,x)$ for 100 backpropagation steps (learning rate $0.01$). The
experiment is repeated on the input-output pair $(1-x,1-x)$. In both
cases all initial weights are set to $0$.}
\end{center}
\end{figure}

It would be interesting to have algorithms whose performance is
insensitive to particular choices such as scaling factors in network
construction, parameter encoding or data representation.
We call an algorithm \emph{invariant}, or \emph{intrinsic}, if applying a
change of variables to the parameters and activities results in the same
learning trajectory. This is not the case for backpropagation (even after
changing the learning rate): for instance, changing from sigmoid to tanh
activation amounts to dividing the connection weights by $4$ and
shifting the biases by half the weights
(Eqs.\ \eqref{eq:sigmtanh1}--\eqref{eq:sigmtanh2}), which does not preserve
gradient directions.

Invariance of an algorithm means fewer arbitrary design choices, and also
more robustness: good performance on a particular problem indicates good
performance over a whole class of problems equivalent to the first one by
simple (e.g., affine) transformations.

Known invariant algorithms include Newton or quasi-Newton
methods \cite{BeckerLeCun88}, or the
natural gradient \cite{Amari98}. The latter, in particular, is invariant (and thus preserves performance)
under a wide class of changes in the representation of the data and of the
network, while Newton-like methods are only invariant under affine
transforms. However, these methods are generally not scalable: the cost of
maintaining the whole Hessian or Fisher information matrix
is quadratic in parameter dimension and prohibitive for large networks \cite{BeckerLeCun88,LBOM96}. The approximations
made to ensure their scalability, such as keeping only diagonal terms
\cite{LBOM96,peskylr},
making small-rank approximations \cite{TONGA}, or using limited-memory BFGS
(e.g.~\cite{NN_BFGS2011} for a recent example), break their invariance properties.

\paragraph{Scalable Riemannian methods for neural networks.} In this work
we introduce four invariant algorithms adapted to four different
scalability constraints. For this we develop a suitable theoretical
framework in which to build invariant algorithms, by treating neural
networks in the context of Riemannian geometry.

For a network with $n$ units, $n\out$ output
units and at most $d$ incoming connections per unit, processing each data sample
with backpropagation has an algorithmic cost $O(nd)$.
The most lightweight of our invariant algorithms has $O(nd)$ cost per
sample as well. The heaviest one is more faithful to the true natural gradient
but has a cost $O(nd^2+ndn\out)$ and thus requires that the network is sparsely
connected (the average number of units influencing a given unit must not
be too large) and that the output layer is not too large. This latter
condition is typically fulfilled in classification tasks.

The \emph{unitwise natural gradient} is a scaled-down version of Amari's
natural gradient \cite{Amari98} in which the blocks of incoming parameters
to each unit are treated independently, thus dealing, at each
unit, with a square matrix indexed by the incoming parameters to this
unit. This has been proposed as far back as \cite{Kur94} to train
neural networks; however the algorithm in \cite{Kur94} is limited to
networks with only one hidden layer, because it relies on an explicit
symbolic computation of entries of the Fisher matrix. Here
Proposition~\ref{prop:fishisfish} allows for an efficient computation of
the \emph{exact} Fisher information matrix in arbitrary neural networks,
by doing $n\out$ distinct backpropagations for each sample in the
dataset. As a result, the unitwise natural gradient is adapted to
situations where both the connectivity of the network and the output
layer are reasonably small: the algorithmic cost of
processing each data sample is $O(nd^2+ndn\out)$.

The \emph{backpropagated metric gradient} 
can be described as a blockwise quasi-Newton method in which several
approximations (Gauss--Newton and neglecting cross-unit terms) are used.
However, we describe it in an intrinsic way: it stems from a well-defined
\emph{backpropagated metric} on parameter space, in which no
approximations are involved. Invariance properties follow
from this viewpoint. It is adapted to networks with reasonably small
connectivity but output layers of arbitrary size: the cost of processing
a data sample is $O(nd^2)$.

The \emph{quasi-diagonal natural gradient} and \emph{quasi-diagonal
backpropagated metric gradient} apply a
``quasi-diagonal reduction'' to these two algorithms, which
removes the quadratic dependency on connectivity at each unit. This is
done in a specific way to keep some (but not all) of the invariance
properties, such as insensitivity to using sigmoid or
$1.7159\tanh(2x/3)$. The quasi-diagonal natural gradient still requires
that the output layer is not too large, with a cost of $O(ndn\out)$ per
data sample, whereas the quasi-diagonal
backpropagated metric gradient has the same $O(nd)$ complexity as ordinary
backpropagation.
These quasi-diagonal methods
have not been described before, to the best of our
knowledge.

In this context, we also clarify another method found in the literature
\cite{APF00,TONGA}. It is related to, and sometimes confused with, the
natural gradient (discussion in \cite{BengioNG2013}). We call this method
the \emph{outer product metric gradient} and introduce a scaled-down,
invariant version. We prove a novel interpretation of this method as the
unique one that, at each step, spreads the improvement most uniformly
over all training samples (Proposition~\ref{prop:equalize}).

\paragraph{Organization of the text.} In Section~\ref{sec:algos} we
give the explicit form of the algorithms, without justification, to serve
as a reference for implementation. In Section~\ref{sec:invmet} we provide
the mathematical principles behind the algorithms, starting with the
relationship between gradients, metrics, and choice of coordinates
(Section~\ref{sec:gradintro}), then using the tools of Riemannian
geometry to build several invariant metrics for neural networks
(Section~\ref{sec:main} and Appendix~\ref{sec:formal}) together with a way of computing them. In
Section~\ref{sec:qd} we introduce the quasi-diagonal reduction of a
metric. These metrics produce associated gradient
descents (Section~\ref{sec:grads}). In Section~\ref{sec:usualfish} we
discuss in detail the case of the Fisher metric for neural networks and
various ways to compute or approximate it. In Section~\ref{sec:someprops}
we present some mathematical properties of these algorithms, focusing on
invariance by change of coordinates (Section~\ref{sec:inv}) and a
``best-fit'' interpretation (Section~\ref{sec:bestfit}). 
Section~\ref{sec:exp} contains a set of small-scale experiments as a proof
of concept for the new algorithms.

A companion article \cite{pcnn} develops related ideas for
\emph{recurrent} neural networks and provides more in-depth experiments
on complex symbolic data sequences.

\bigskip

We now provide an introduction to how invariant algorithms are built, and
an overview of the experimental results.

\paragraph{Gradient descents and metrics.} To build these invariant
algorithms, we use gradient descent in suitable invariant \emph{metrics}.

Backpropagation is the simple gradient descent over parameter space.
Gradient descents follow the steepest direction of change in parameter
space, and implicitly rely on a norm (or quadratic form, or metric) to
define the steepest direction:
the gradient step
$x\gets x-\eta\,\nabla f$ can also be rewritten (for small enough
$\eta$, up to $O(\eta^2)$ and for regular enough functions $f$) as
\begin{equation}
x\gets \argmin_y \left\{ f(y)+\frac{1}{2\eta}\norm{y-x}^2\right\}
\end{equation}
namely, the gradient descent moves into the direction yielding the
smallest values of $f$, penalized by the distance from the current
point, measured by the square norm $\norm{y-x}^2$.
For
backpropagation this norm $\norm{\cdot}^2$ is the numerical change in the values of the
parameters: backpropagation provides the direction of largest
improvement for a minimal change in these numerical values. Hence
simple changes in parametrization influence the behavior of the
algorithm. On the other hand, norms $\norm{\cdot}^2$ based on what the network does,
rather than how it is represented as numbers, will lead to ``intrinsic''
algorithms. This is one of the ideas behind Amari's natural gradient.

In Section~\ref{sec:invmet} we build several invariant norms, by placing
neural networks in the context of differential manifolds and Riemannian
geometry. The gradient descent coming from an invariant norm (Riemannian
metric) will itself be invariant. Moreover, any gradient descent using
any norm has the property that small enough learning rates ensure
performance improvement at each step.

The resulting algorithms are all invariant under a number of
transformations, including affine reparametrization
of the unit activities. Among the invariance properties enjoyed by the
unitwise natural gradient and the backpropagated metric (but not their
quasi-diagonal reductions) are linear recombinations of the input
received by a given unit in the network, so that
a unit receiving signals $f$ and $f+\eps g$ (as functions over the
dataset) from incoming units will learn an output correlated to
$g$ just as fast as a unit receiving signals $f$ and $g$ (on the input layer
this can be accounted for by normalizing and de-correlating the inputs,
but this could occur at internal units as well). Thus these gradients have a
``best-fit'' interpretation (Section~\ref{sec:bestfit}): at each unit
they solve a least-square problem of matching input signals and desired
backpropagated output, an interpretation proposed in \cite{Kur94}.

The quasi-diagonal reductions of these algorithms are based on the
observation that there is a distinction between
weights $w_{ik}$ and $w_{jk}$ coming to $k$ from different units, but no
intrinsic mathematical separation between weights and biases.
Intuitively, given that unit $k$ receives a signal $w_{ik}a_i$ from unit $i$,
if we change $w_{ik}$ to $w_{ik}+\d w_{ik}$, the average signal to
unit $k$ will change by $\d w_{ik} \bar a_i$ where $\bar a_i$ is the
average activation of $i$. Hence it might be a good idea to automatically
add $-\d w_{ik} \bar a_i$ to the bias of $k$, to compensate. The
quasi-diagonal algorithms we present are more sophisticated versions of
this, tuned for invariance and using weighted averages. The few added
terms in the update sometimes greatly improve performance
(Fig~\ref{fig:diaghess} on page~\pageref{fig:diaghess}).

Arguably, none of these algorithms is second-order: the update on
parameter $\theta$ takes the form $\theta\gets
\theta-A(\theta)^{-1}\,\partial_\theta f$ where $A(\theta)$ is a matrix
depending on the network but not on the objective function $f$. This
matrix comes from (Riemannian)
metrics evaluating the magnitude of the effect on the output of changes
in a given direction, thus providing a suitable learning rate for each
direction, without estimating derivatives of gradients. 
Second-order effects are emulated in the same way the Gauss--Newton algorithm emulates
the Newton method\footnote{Actually, in the framework of differential
geometry, without a metric, the Hessian is only defined at local optima
of the function \cite[paragraph 3.37]{GHL87}, so one could say that in
such a setting the Newton method approximates the Gauss--Newton algorithm
rather than the other way around.}.

\paragraph{Experimental proof of concept.} While the focus of this article is
mainly the mathematics of neural network training, we quickly tested
experimentally the implementability and impact of using the various methods. We selected  a
very simple auto-encoding problem on which we expected that any training
method would perform well. A sparsely connected network with $5$ layers
of size $100$, $30$, $10$, $30$, and $100$ was built, and $16$ random
length-$100$ binary strings were fed to the input layer, with the target
equal to the input. Ideally the network learns to encode each of the $16$
samples using $4$ bits on the middle layer (thus with room to spare) and
rewrites the output from this.  The details are given in
Section~\ref{sec:exp}.

Even then, backpropagation performs poorly: after $10,000$ batch passes
the average log-loss is about $36$ bits per sample (out of $100$) for
sigmoid backpropagation, and about $25$ bits per sample for tanh
backpropagation. Note that $30$ bits per sample would correspond to a
method which learns only the parameters of the output layer and can
reproduce the output if someone fills the last hidden layer with the
correct $30$ bits.

\begin{figure}
\begin{center}
\includegraphics[width=.9\columnwidth]{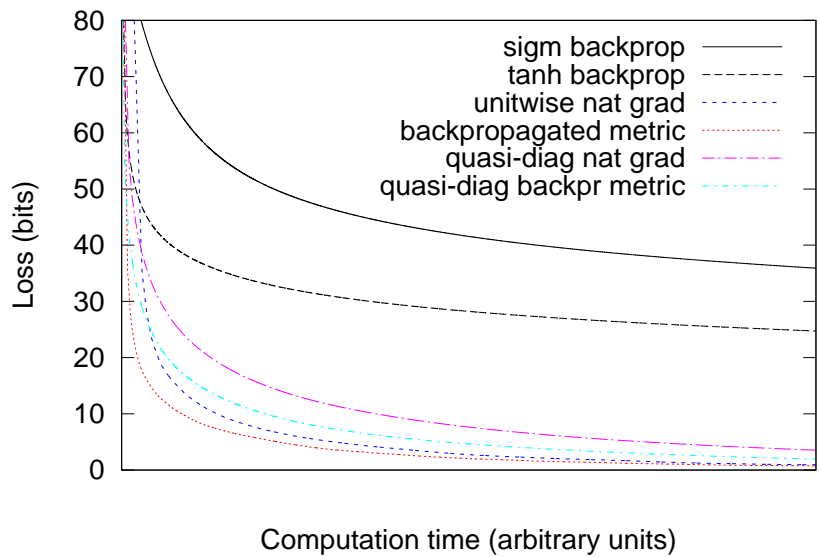}

\caption{\label{fig:mainfig}
Auto-encoding using a 100--30--10--30--100 deep sparsely connected network.
Comparison of backpropagation using sigmoid and tanh activation, and the
four algorithms described in Section~\ref{sec:algos}, for a given
computation time budget.}
\end{center}
\end{figure}

For the \emph{same} total computation time equivalent to $10,000$ batch
backpropagations\footnote{as measured in CPU
time on a personal computer, but this can depend a lot on implementation
details}, the quasi-diagonal algorithms have a log loss of about $1.5$ to
$3.5$ bits per sample, and both the unitwise natural gradient and the
backpropagated metric gradient have a log loss of $0.3$ to $1.5$ bit per
sample, thus essentially solving the problem. See
Figure~\ref{fig:mainfig}.

The impact of the few non-diagonal terms in the quasi-diagonal
algorithms was tested by removing them. In this case the quasi-diagonal
backpropagated metric gradient reduces to the diagonal Gauss--Newton
method (\cite[Section 7.4]{LBOM96}, also used in \cite{peskylr}). This breaks the invariance
properties, thus the impact is different for sigmoid or tanh
implementations.  The diagonal Gauss--Newton method in sigmoid
implementation was found to perform more poorly, with a final log-loss
of about $12$ bits per sample (Figure~\ref{fig:diaghess} on
page~\pageref{fig:diaghess}), while in tanh implementation it comes
somewhat close to
the quasi-diagonal algorithms at about $3.5$ bits per sample
(presumably because in our problem, the activity of all units, not only
input units, stay perfectly centered during training). Thus the
quasi-diagonal backpropagated metric gradient can be seen as ``the
invariant way'' to write the diagonal Gauss--Newton method, while
performance of the latter is not at all invariant.

We also compared the exact unitwise natural gradient obtained thanks to
Proposition~\ref{prop:fishisfish}, to a variant of the natural gradient
in which only the gradient terms $b\transp{b}$ corresponding to the target
for each sample are added to the Fisher matrix (\cite{APF00,TONGA} and
Section~\ref{sec:usualfish} below). The latter,
when implemented unitwise, performs rather poorly on this auto-encoding
task, with a log loss of about $25$ to $28$ bits per sample. The reason,
discussed in Section~\ref{sec:exp}, may be that quality of this
approximation to the Fisher matrix strongly depends on output
dimensionality.

One lesson from the numerical experiments is that the regularization
term $\eps \Id$ added to the matrices, needed to prevent bad
behavior upon inversion, formally breaks the invariance properties:
individual trajectories in sigmoid or tanh implementations, initialized
in the same way, start to differ after a dozen iterations. Still, overall
performance is not affected and is the same in both implementations
(Figure~\ref{fig:invtraj}, p.~\pageref{fig:invtraj}).

Though the quasi-diagonal methods perform well, the only methods to
sometimes reach very small values of the loss function on this example
(less than $0.1$ bit per sample) are the unitwise natural
gradient and the backpropagated metric, which at each unit maintain a
full matrix over the incoming parameters and thus achieve invariance
under affine recombination of incoming signals. These two methods are
relevant only when network connectivity is not too high. This highlights
the interest of sparsely connected networks from a theoretical viewpoint.

\subsection*{Notation for neural networks}

Consider a directed neural network model: a set $\L$ of units together
with a set of directed edges $i\to j$ for $i,j\in\L$, without cycle. Let
$\Lout$ be the output units, that is, the units with no outgoing
edges\footnote{This restriction on output units is not necessary (any
unit could be an output unit) but
simplifies notation.},
and similarly let $\Lin$ be the set of units without incoming edges.
Let $s:\R \to \R$ be an \emph{activation function}. 
Given an activation level for the input
units, each unit $j$ gets an
activation level
\begin{equation}
a_j=\actf\left(\ssum_{i\to j} w_{ij} a_i\right)
\end{equation}
depending on the activation levels of the units $i$ pointing to $j$ and on the firing coefficients
$w_{ij}$ from\footnote{What is $w_{ij}$ for some authors is $w_{ji}$ for
others. Our convention is the same as \cite{RN2} but for instance \cite{LBOM96} follows the
opposite convention.} $i$ to $j$.
\emph{Biases} are treated as the weights $w_{0j}$ from
a special always-activated unit $0$ ($a_0\equiv 1$) connected to every
other unit. A common choice
for the activation function $s$ is the \emph{logistic} \cite{RN2}
function
$\actf(V)=\frac{\e^V}{1+\e^V}=\frac{1}{1+\e^{-V}}$, although for instance
\cite{LBOM96} recommends the hyperbolic tangent
$\actf(V)=\tanh(V)$, related to the logistic by
$\tanh(V)=2(\frac{1}{1+\e^{-2V}})-1$.
We refer to
\cite{RN2}, which we mostly follow with minor changes in notation.

For a given
non-input unit
$j$, we call the parameters $w_{0j}$ and $w_{ij}$ for $i\to j$ the set of
\emph{incoming parameters} to unit $j$.

The dataset for this network is a set $\D$ of inputs, where each input
$x\in \R^{\Lin}$ is a real-valued vector over the input layer. For each
input is given a target $y$ in an arbitrary space.
We view the network as a probabilistic generative model: 
given an input $a_i=x_i$ on the input layer $\Lin$, we assume that the
activations of the output layer are interpreted in a fixed way as a
probability distribution $\omega(y)$ over the target space. The goal is
to maximize the probability to output $y$ on input $x$: we define the
loss function
\begin{equation}
\ell(\omega,y)\deq -\ln \omega(y)
\end{equation}
the sum of which over the dataset is to be minimized. For instance,
interpreting the output layer activities $(a_k)_{k\in\Lout}$ as Gaussian
variables with mean $a_k$ and variance $1$ leads to a quadratic loss
function $\ell$.


\begin{ex}[ (Square-loss, Bernoulli, and two classification
interpretations)]
The \emph{square-loss interpretation} of the output layer sends the
activations $(a_k)_{k\in\Lout}$ of the output layer to a random variable
$Y=(Y_k)_{k\in \Lout}$ of independent Gaussian variables, where each $Y_k\sim
\mathcal{N}(a_k,1)$ is a Gaussian of mean $a_k$.

Assume that the activities $a_i$ of units in the network lie in $[0;1]$.
The \emph{Bernoulli interpretation} of the output layer is a Bernoulli
distribution as follows:
given the activations $(a_k)_{k\in\Lout}$ of the output layer, the final
output is
a $\{0,1\}^{\Lout}$-valued random variable
$Y=(Y_k)_{k\in
\Lout}$ of independent Bernoulli variables, where the activations $a_k$ give the probability
to have $Y_k=1$, namely $\Pr(Y_k=1)=a_k$.

For classification, the interpretation must send the output activations
$(a_k)_{k\in\Lout}$ to a probability distribution over
the indices $k$ in the output layer. In the \emph{softmax
interpretation}\footnote{Usually combined
with a linear activation function ($\actf=\Id$) on the last layer}, the probability of class $k\in\Lout$ is
$\e^{a_k}/\sum_{k'\in\Lout} \e^{a_{k'}}$. In the \emph{spherical
interpretation}\footnote{This latter example is motivated by a theoretical argument:
the set of probability
distributions over a finite set, equipped with its Fisher metric, is
isometric to the positive quadrant in a sphere and so is naturally
parametrized by numbers $a_k$ with $\sum a_k^2=1$, and these variables
yield a slightly simpler expression for the Fisher matrix. Besides, taking squares
gives a boost to the most activated output unit, in a smooth way, as in
the softmax interpretation.},
the probability of class $k\in
\Lout$ is $a_k^2/(\sum_{k'\in\Lout} a_{k'}^2)$.
\end{ex}

Remark~\ref{rem:outputparams} covers the case when the interpretation
depends on extra parameters $\vartheta$, such as a softmax
$\Pr(k)=\e^{\vartheta_k a_k}/\sum_{k'\in\Lout} \e^{\vartheta_{k'} a_{k'}}$
with trainable weights $\vartheta_k$.

\paragraph{Backpropagation.} A common way to train the network on a given target value $y$ is by
backpropagation, which amounts to a gradient descent over the parameters
$w_{ij}$. For a given loss function, define the backpropagated value $b_i$ 
at each unit
$i$ by
\begin{equation}
\label{eq:b}
b_i \deq -\frac{\partial \ell}{\partial a_i}
\end{equation}
so that $b_i$ indicates how we should 
modify the activities to decrease $\ell$. Then the value of $b_i$ satisfy
the \emph{backpropagation equations} \cite{RN2} from the output layer:
\begin{equation}
b_i= \sum_{j,\,i\to j} w_{ij}\deractf_j b_j
\qquad \text{for $i\not\in \Lout$}
\end{equation}
where the activation levels $a_i$ have first been computed by forward
propagation, and where let $r_j$ stand for the value of the derivative of
the activation function at unit $j$:
\begin{equation}
\label{eq:r}
\deractf_j\deq \actf'\left( \sum_{i\to j} w_{ij}a_i\right)=
\begin{cases}
a_j(1-a_j) & \text{for sigmoid activation function}
\\
1-a_j^2 & \text{for tanh activation function}
\end{cases}
\end{equation}
The backpropagated values on the output layer are computed directly by
\eqref{eq:b}, for instance,
\begin{equation}
b_i=
\begin{cases}
y_i-a_i & \text{(square-loss interpretation)}
\\
\frac{y_i-a_i}{a_i(1-a_i)} & \text{(Bernoulli interpretation)}
\\
y_i-\frac{\e^{a_i}}{\sum_{k\in\Lout}\e^{a_k}} &
\text{(softmax interpretation)}
\\
\frac{2y_i}{a_i}-\frac{2a_i}{\sum_{k\in\Lout}a_k^2} &
\text{(spherical interpretation)}
\end{cases}
\end{equation}
for $i\in\Lout$.

The backpropagated values are used to compute the gradient of the loss
function with respect to the parameters $w_{ij}$. Indeed we have
\begin{equation}
\frac{\partial \ell}{\partial w_{ij}}=-a_i\deractf_jb_j
\end{equation}
for each edge $(ij)$ in the network. (This includes the bias $w_{0j}$
using $a_0\equiv 1$.)

It is sometimes more convenient to work with the reduced
variables 
\begin{equation}
\tilde b_i\deq \deractf_ib_i
\end{equation}
which satisfy the
backpropagation equation
\begin{equation}
\label{eq:btilde}
\tilde b_i=\deractf_i\sum_{j,\,i\to j} w_{ij} \tilde b_j
\qquad \text{for $i\not\in \Lout$}
\end{equation}
and
\begin{equation}
\label{eq:gradtilde}
\frac{\partial \ell}{\partial w_{ij}}=-a_i\tilde b_j
\end{equation}

The gradient descent with learning rate $\eta>0$ is then defined as the
following update on the firing coefficients:
\begin{equation}
w_{ij}\gets w_{ij}- \eta\frac{\partial \ell}{\partial w_{ij}}=w_{ij}+\eta\,
a_i\deractf_j b_j= w_{ij}+\eta\,
a_i \tilde b_j
\end{equation}


\section{Four invariant gradient algorithms}
\label{sec:algos}

We now describe four gradient algorithms for network training: the
\emph{unitwise natural gradient}, the \emph{quasi-diagonal natural
gradient}, the \emph{backpropagated metric gradient}, and the
\emph{quasi-diagonal backpropagated metric gradient}. Each of these
algorithms is adapted to a different scalability constraint. The unitwise
natural gradient requires low connectivity and a small output layer; the
quasi-diagonal natural gradient requires a small output layer; the
backpropagated metric gradient requires low connectivity; the
quasi-diagonal backpropagated metric gradient has the same asymptotic
complexity as backpropagation.

These algorithms are the implementation of the more general versions
described in Section~\ref{sec:invmet}. As they are designed for
invariance properties, implementing them using either sigmoid or tanh
activation function should result in the same output, learning
trajectories, and performance, provided the initialization is changed
accordingly (Section~\ref{sec:inv}). However, the Bernoulli and classification interpretations of the output layer
assumes that the activities lie in $[0;1]$, as in
sigmoid activation.

We first present the algorithms in a batch version. It is straightforward
to adapt them to use random mini-batches from the dataset. In
Section~\ref{sec:online} they are also adapted to an online setting: this can
be done using standard techniques because the main quantities involved
take the form of averages over the dataset, which can be updated online.

\subsection{Unitwise natural gradient and quasi-diagonal natural
gradient}

The unitwise natural gradient has been proposed as far back as
\cite{Kur94} to train neural networks; however the presentation in
\cite{Kur94} is limited to networks with only one hidden layer, because
it relies on an explicit symbolic computation of entries of the Fisher
matrix. Proposition~\ref{prop:fishisfish} below allows for an efficient
computation of the \emph{exact} Fisher information matrix by doing $n\out$ distinct
backpropagations for each sample in the dataset. This relies on linearity
of backpropagation, as follows.

\begin{defi}[ (Backpropagation transfer rates)]
\label{def:trates}
Fix an input $x$ for the network and compute the activities by forward
propagation.
Let $k$ be a unit in the network and $k\out$ be a unit in the output
layer. Define the \emph{backpropagation transfer rates} $J_k^{k\out}$
from $k\out$ to $k$ by backpropagating the value $1$ at
$k\out$. Formally:
\begin{equation}
\begin{cases}
J^{k\out}_{k\out}\deq 1\quad,\qquad
J^{k\out}_k\deq 0,
& \text{for $k\neq k\out$ in the output layer $\Lout$}
\\
J^{k\out}_k\deq \sum_{j,\,k\to j} w_{kj} \,\deractf_j\,J^{k\out}_j
& \text{for non-output units $k$}
\end{cases}
\end{equation}
where $\deractf_j$ is the derivative of the activation function, given by \eqref{eq:r}.
\end{defi}

These transfer rates have the property that if backpropagation
values $b$ are set on the output layer, then $b_k=\sum_{k\out\in \Lout}
J_k^{k\out} b_{k\out}$ for any unit $k$ (see also \cite[Section 7.2]{LBOM96}).

Computation of the transfer rates can be done by $n\out$ distinct
backpropagations. There are further simplifications, since the transfer
rates for $k$ in the input layer are never used (as there are no incoming
parameters), and the transfer rates on the last hidden layer are readily
computed as $J^{k\out}_k=w_{kk\out} \deractf_{k\out}$. Thus it is
enough to backpropagate the transfer rates from the last hidden layer to
the first hidden layer. In particular,  with only one hidden layer (the case
considered in \cite{Kur94} for the Fisher matrix) no backpropagation is
needed.

\begin{defi}[ (Fisher modulus)]
\label{def:fishmod}
Fix an input $x$ for the network and compute the activities by forward 
propagation. For each unit $k$ in the network, define the \emph{Fisher
modulus} $\Phi_k(x)$ of unit $k$ on input $x$ as follows, depending on output layer
interpretation.
\begin{itemize}
\item For the Bernoulli interpretation, set
\begin{equation}
\Phi_k(x)\deq \sum_{k\out\in \Lout}
\frac{(J_k^{k\out})^2}{a_{k\out}(1-a_{k\out})}
\end{equation}

\item For the square-loss interpretation, set
\begin{equation}
\Phi_k(x)\deq \sum_{k\out\in \Lout} (J_k^{k\out})^2
\end{equation}

\item For the softmax interpretation, set
\begin{equation}
\Phi_k(x)\deq \frac{1}{S}\sum_{k\out\in\Lout}\e^{a_{k\out}}(J_k^{k\out})^2-\frac{1}{S^2}
\left(\sum_{k\out\in\Lout} \e^{a_{k\out}}J_k^{k\out}\right)^2
\end{equation}
where $S\deq\sum_{k\out\in\Lout} \e^{a_{k\out}}$.

\item For the spherical interpretation, set
\begin{equation}
\Phi_k(x)\deq \frac{4}{S}\sum_{k\out\in\Lout}(J_k^{k\out})^2-\frac{4}{S^2}
\left(\sum_{k\out\in\Lout} a_{k\out}J_k^{k\out}\right)^2
\end{equation}
where $S\deq\sum_{k\out\in\Lout} a_{k\out}^2$.
\end{itemize}
\end{defi}

\begin{defi}[ (Unitwise Fisher matrix)]
Let $k$ be a unit in the network. Let $E_k$ be the set of incoming units 
to $k$ (including the always-activated unit $0$). The \emph{unitwise
Fisher matrix} at unit $k$ is the $(\#E_k)\times (\#E_k)$ matrix $F^{(k)}$ defined by
\begin{equation}
F^{(k)}_{ij}\deq \E_{x\in \D} \, a_ia_j \deractf_k^2 \Phi_k
\end{equation}
for $i$ and $j$ in $E_k$ (including the unit $0$ with $a_0\equiv 1$),
with $\Phi_k$ the Fisher modulus of Definition~\ref{def:fishmod}.
Here $\E_{x\in \D}$ represents the average
over samples $x$ in the dataset (all the terms $a_i$, $a_j$, $\deractf_k$,
$\Phi_k$ depend on the input to the network). 
\end{defi}

By Proposition~\ref{prop:fishisfish} below, $F^{(k)}_{ij}$ is
the block of the Fisher information matrix associated with the
incoming parameters to $k$, hence the name.

\begin{defi}[ (Unitwise natural gradient)]
\label{def:unatgrad}
The \emph{unitwise natural gradient} with learning rate $\eta>0$ updates the parameters of the
network as follows.

For each unit $k$, define the vector $G^{(k)}$ by
\begin{equation}
G^{(k)}_i\deq -\E_{x\in \D}\frac{\partial \ell(y)}{\partial w_{ik}}
=\E_{x\in \D} \,a_i \deractf_kb_k
\label{eq:defGk}
\end{equation}
where $b_k$ is the backpropagated value at $k$ obtained from the target
$y$ associated with $x$, and where $i$ runs over the incoming units to
$k$ (including the always-activated unit $i=0$).
Compute the vector
\begin{equation}
\label{eq:locnatgrad}
\d w^{(k)}\deq (F^{(k)})^{-1} G^{(k)}
\end{equation}

Then the parameters of the network are updated by
\begin{equation}
\label{eq:locnatstep}
w_{ik}\gets w_{ik}+\eta \d w_i^{(k)}
\end{equation}
\end{defi}

The asymptotic algorithmic cost of the unitwise natural gradient is as
follows. Computing $\tau$ requires $n\out$ distinct backpropagations for
each input $x$. For
a network with $n$ units, $n\out$ output units, and at most $d$ incoming
connections per unit, this costs $O(n\out nd)$ per data sample (this can be reduced in some
cases, as discussed above). Computing $\Phi$ takes $O(nn\out)$. Computing
$F^{(k)}_{ij}$ for every $k$ takes $O(nd^2)$.  Computing the gradient
requires inverting these matrices, which takes $O(nd^3)$ but is done only
once in each (mini-)batch, thus if the size of batches is larger than $d$
this cost is negligible; if the size of batches is smaller than $d$ or in
an online setting, inversion can be done recursively using the Sherman--Morrison
formula at a cost of $O(nd^2)$ per data sample. So the overall cost of
this gradient is $O(nd^2+n\out nd)$ per data sample.

This algorithmic cost is fine if connectivity is low. We now define a
more light-weight version in case connectivity is large. Its
computational cost is equivalent to that of ordinary backpropagation
provided the output layer is small.

\begin{defi}[ (Quasi-diagonal natural gradient)]
The \emph{quasi-diagonal natural gradient} with learning rate $\eta>0$ updates the parameters of the
network as follows.

For each unit $k$, compute only the entries $F_{00}^{(k)}$,
$F_{0i}^{(k)}$, and $F_{ii}^{(k)}$ of the unitwise Fisher matrix at $k$.
Define the vector $G^{(k)}$ as in~\eqref{eq:defGk} above. Define the
vector $\d w^{(k)}$ by
\begin{align}
\label{eq:qdfishupdate1}
\d w_i^{(k)}&=\frac{G_i^{(k)}F_{00}^{(k)}-G_0^{(k)}F_{0i}^{(k)}}{F_{ii}^{(k)}F_{00}^{(k)}-(F_{0i}^{(k)})^2}
\qquad\text{ for $i\neq 0$}
\\
\label{eq:qdfishupdate2}
\d w_0^{(k)}&=\frac{G_0^{(k)}}{F_{00}^{(k)}}-\sum_{i\neq
0}\frac{F_{0i}^{(k)}}{F_{00}^{(k)}}\d w_i^{(k)}
\end{align}
and update the parameters of the network by
\begin{equation}
w_{ik}\gets w_{ik}+\eta \d w_i^{(k)}
\end{equation}
\end{defi}

With respect to the unitwise natural gradient, the algorithmic cost does
not involve any $O(nd^2)$ terms because we only compute $O(d)$ entries of
the matrices $F$ and do not require a matrix inversion. The variables
$\tau$ and $\Phi$ still need to be computed. Thus the overall complexity
is reduced to $O(n\out n d)$.

The quasi-diagonal formulas may seem arbitrary. If we remember that (omitting the
superscript $(k)$) $F_{00}=\E_{x\in \D} \, \deractf_k^2\,\Phi_k$,
$F_{0i}=\E_{x\in \D} \, a_i \deractf_k^2\,\Phi_k$, and $ F_{ii}=\E_{x\in
\D} \, a_i^2 \,\deractf_k^2\,\Phi_k$, we can consider these sums as
expectations over the dataset  with weights $\deractf_k^2\Phi_k$. Then
the weighted average of $a_i$ is $A_i=F_{0i}/F_{00}$ and its weighted
variance is $V_i=F_{ii}/F_{00}-A_i^2$ so that we have
\begin{equation}
\d w_i=\frac{G_i-G_0A_i}{F_{00}V_i}
\end{equation}
and in particular the denominator is always positive unless the activity
of
unit $i$ is constant (in this case, the numerator vanishes too).

A possible interpretation is as follows: If the activity $a_i$ of unit
$i$ is not centered on average over the dataset (with the weights above),
increasing a weight $w_{ik}$ not only increases the influence of unit $i$
over unit $k$, but also shifts the average activity of unit $k$, which
may not be desirable. Using the method above, if $a_i$ is not centered, when
we change $w_{ik}$ a corresponding term is automatically subtracted from
the bias $w_{0k}$ so as not to shift the average activity of unit $k$, as
discussed in the introduction. On the other hand if the activity $a_i$
is centered, then the
update is diagonal, and scaled by the inverse ``variance'' $1/V_i$.

\subsection{Backpropagated metric gradient and quasi-diagonal
backpropagated metric gradient}

Computing the Fisher matrix as above requires performing $n\out$ backpropagations
for each sample. If one tries to compute the Fisher modulus $\Phi$
directly by backpropagation, the backpropagation equation involves
cross-terms between different units. Neglecting these cross-terms results
in a simpler version of the Fisher modulus which can be computed in one
backward pass; the corresponding backpropagation equation is well-known
as an approximation of the Hessian \cite[Section~7]{LBOM96}. It turns out
this quantity and the associated metric are still intrinsic.

\begin{defi}[ (Backpropagated modulus)]
\label{def:bpmod}
Fix an input $x$ for the network and compute the activities by forward
propagation.
Define the \emph{backpropagated modulus} $m_k(x)$ for each unit $k$ by
\begin{equation}
m_k(x)\deq \sum_{j,\, k\to j} w_{kj}^2 \,\deractf_j^2 \, m_j(x) 
\end{equation}
if $k$ is not an output unit, and, depending on output interpretation,
\begin{equation}
m_k(x)\deq
\begin{cases}
\frac{1}{a_k(1-a_k)}& \text{(Bernoulli)}
\\
1 & \text{(square-loss)}
\\
\frac{\e^{a_k}}{S}(1-\frac{\e^{a_k}}{S}), \quad S=\sum_{k\out\in\Lout}
\e^{a_{k\out}}
& \text{(softmax)}
\\
\frac{4}{S}(1-\frac{a_k^2}{S}), \quad S=\sum_{k\out\in\Lout} a_{k\out}^2
& \text{(spherical)}
\end{cases}
\end{equation}
for $k$ in the output layer.
\end{defi}

\begin{defi}[ (Backpropagated metric)]
Let $k$ be a unit in the network. Let $E_k$ be the set of incoming units 
to $k$ (including the always-activated unit $0$). The
\emph{backpropagated metric} at unit $k$ is the $(\#E_k)\times (\#E_k)$ matrix $M^{(k)}$ defined by
\begin{equation}
M^{(k)}_{ij}\deq \E_{x\in \D} \, a_ia_j \deractf_k^2 m_k
\end{equation}
for $i$ and $j$ in $E_k$ (including the unit $0$ with $a_0\equiv 1$).
Here $\E_{x\in \D}$ represents the average
over samples $x$ in the dataset (all the terms $a_i$, $a_j$, $a_k$,
$m_k$ depend on the input to the network). 
\end{defi}

The backpropagated metric gradient can thus be described as an
approximate, blockwise Hessian method in which the Hessian is
approximated by the Gauss--Newton technique with, in addition, cross-unit
terms neglected. Such a method turns out to be intrinsic.

\begin{defi}[ (Backpropagated metric gradient)]
\label{def:bpmgrad}
The \emph{backpropagated metric gradient} with learning rate $\eta>0$ updates the parameters of the
network as follows.

For each unit $k$, define the vector $G^{(k)}$ by
\begin{equation}
G^{(k)}_i\deq -\E_{x\in \D}\frac{\partial \ell(y)}{\partial w_{ik}}
=\E_{x\in \D} \,a_i \deractf_k b_k
\end{equation}
where $b_k$ is the backpropagated value at $k$ obtained from the target
$y$ associated with $x$, and where $i$ runs over the incoming units to
$k$ (including the always-activated unit $i=0$).
Compute the vector
\begin{equation}
\label{eq:bpmgrad}
\d w^{(k)}\deq (M^{(k)})^{-1} G^{(k)}
\end{equation}

Then the parameters of the network are updated by
\begin{equation}
w_{ik}\gets w_{ik}+\eta \d w_i^{(k)}
\end{equation}
\end{defi}

The algorithmic cost of the backpropagated metric gradient is $O(nd^2)$
per data sample, with notation as above. Indeed, computing $m$ costs the
same as a backpropagation pass, namely $O(nd)$. Computing the matrices
$M$ costs $O(nd^2)$. Inverting the matrices has no impact on the overall
complexity, as explained after
Definition~\ref{def:unatgrad}. This cost is acceptable for small
$d$ (sparsely connected networks). For large $d$ we define the following.

\begin{defi}[ (Quasi-diagonal backpropagated metric gradient)]
The \emph{quasi-diagonal backpropagated metric gradient} with learning rate $\eta>0$ updates the parameters of the
network as follows.

For each unit $k$, compute only the entries $M_{00}^{(k)}$,
$M_{0i}^{(k)}$, and $M_{ii}^{(k)}$ of backpropagated metric at $k$.
Define the vector $G^{(k)}$ as in~\eqref{eq:defGk} above. Define the
vector $\d w^{(k)}$ by
\begin{align}
\label{eq:qdbpmupdate1}
\d w_i^{(k)}&=\frac{G_i^{(k)}M_{00}^{(k)}-G_0^{(k)}M_{0i}^{(k)}}{M_{ii}^{(k)}M_{00}^{(k)}-(M_{0i}^{(k)})^2}
\qquad\text{ for $i\neq 0$}
\\
\label{eq:qdbpmupdate2}
\d w_0^{(k)}&=\frac{G_0^{(k)}}{M_{00}^{(k)}}-\sum_{i\neq
0}\frac{M_{0i}^{(k)}}{M_{00}^{(k)}}\d w_i^{(k)}
\end{align}
and update the parameters of the network by
\begin{equation}
w_{ik}\gets w_{ik}+\eta \d w_i^{(k)}
\end{equation}
\end{defi}

The same remarks as for the quasi-diagonal natural gradient apply for
interpreting the various terms. The
denominator $M_{ii}^{(k)}M_{00}^{(k)}-(M_{0i}^{(k)})^2$ can be seen as a
weighted variance of the activity of unit $i$, and is positive unless
$a_i$ is constant over the dataset. The contribution of $\d w_i^{(k)}$ to
$\d w_0^{(k)}$ compensates the change of average activity induced by a
change of $w_{ik}$.

The asymptotic cost of this update is $O(nd)$ per data sample,
as for backpropagation. 

If, in the quasi-diagonal backpropagated metric gradient, the
non-diagonal terms are neglected ($M^{(k)}_{0i}$ is set to $0$), then
this reduces to the diagonal Gauss--Newton method equations from \cite[Section
7.4]{LBOM96} (also used for instance in \cite{peskylr}).

\begin{rem}
On the incoming parameters to the output layer, the unitwise natural
gradient and the backpropagated metric gradient coincide if the
Bernoulli or square-loss interpretation is used. (Actually, with learning
rate $\eta=1$ they also both
coincide with the Newton method restricted to the output layer
parameters.)
\end{rem}

\begin{rem}
Since these algorithms rely on inverting matrices, regularization is an
issue. In practice, terms $\eps\Id$ have to be added to $F$ and $M$
before inversion; terms $\eps$ have to be added to the diagonal terms
$F^{(k)}_{00}$, $F^{(k)}_{ii}$, $M^{(k)}_{00}$ and $M^{(k)}_{ii}$ in
the quasi-diagonal reduction. This formally breaks the invariance
properties. Section~\ref{sec:bestfit} elaborates on this. Still, this
operation preserves the guarantee of improvement for small enough
learning rates. 
\end{rem}

\subsection{Adaptation to an online setting}
\label{sec:online}

The unitwise natural gradient and unitwise backpropagated metric gradient
both update the weights by
\begin{equation}
\delta w=A^{-1}G
\end{equation}
with $G$ the gradient of
the loss function over the dataset, and $A$ a positive-definite, symmetric
matrix. A key feature here is that the matrix $A$ takes the form of an expectation over the dataset:
$F^{(k)}_{ij}= \E_{x\in \D} \, a_ia_j\deractf_k^2 \Phi_k$ for the
Fisher matrix, and $M^{(k)}_{ij}= \E_{x\in \D} \, a_ia_j\deractf_k^2
m_k$ for the backpropagated metric.

Any such algorithm can be turned online using a standard construction as
follows
(compare e.g.\ \cite{TONGA}). Another possibility is, of course, to use
mini-batches.

In the following, $A$ stands for either the unitwise Fisher matrix or the
backpropagated metric. Let $A(x)$ be the corresponding contribution of
each input $x$ in the expectation, namely, $A(x)^{(k)}_{ij}=a_ia_j\deractf_k^2
\Phi_k$ for the Fisher metric and $A(x)^{(k)}_{ij}=a_ia_j\deractf_k^2
m_k$ for the backpropagated metric, so that $A=\E_{x\in \D} \,A(x)$.

At each step $t$, we use one new sample in the dataset,
update an estimate $A^{(t)}$ of $A$, and follow a gradient step for this
sample, as follows.

\begin{itemize}
\item Initialize the matrix $A^{(0)}$ by using a small subsample
$\D_\mathrm{init}\subset \D$, for instance the first $n_\mathrm{init}$
samples in the dataset:
\begin{equation}
A^{(0)}\deq \E_{x\in \D_\mathrm{init}} \,A(x)
\end{equation}
\item Fix a discount factor $0<\gamma<1$. For each new sample $x_t$,
compute its contribution $A(x_t)$ and update $A$ by
\begin{equation}
A^{(t)}\deq(1-\gamma)A^{(t-1)}+\gamma A(x_t)
\end{equation}
\item Compute the inverse of $A^{(t)}$ from the inverse of $A^{(t-1)}$
by the Sherman--Morrison formula at each unit, using that $A(x_t)$ is a
rank-one matrix at each unit. (This way matrix inversion is no more costly
than the rest of the step.)
\item Compute the negative gradient $G(x_t)$ of the loss function on input $x_t$
by backpropagation.
\item Update the parameters by
\begin{equation}
w\gets w+\eta_t (A^{(t)})^{-1} G(x_t)
\end{equation}
where $\eta_t$ is the learning rate.
\item For the quasi-diagonal reductions of the algorithms, only the entries
$A_{00}$, $A_{ii}$ and $A_{0i}$ of the matrix $A$ are updated at each
step. No matrix inversion is required for the update
equations~\eqref{eq:qdfishupdate1}--\eqref{eq:qdfishupdate2} and
\eqref{eq:qdbpmupdate1}--\eqref{eq:qdbpmupdate2}.
\end{itemize}

We could also initialize $A^{(0)}$ to
a simple matrix like $\Id$, but this breaks the invariance properties of
the algorithms.

The update rule for $A^{(t)}$ depends on the discount factor $\gamma$. It should be large enough
so that a large number of data points contribute to the computation of
$A$, but small enough to be reactive so that $A$ evolves as training gets
along. In our setting, from the particular form of $A(x)$ at a unit $k$
we see that each $A(x_t)$ contributes a rank-one matrix. This means that
$\gamma$ should be much smaller than $1/n_k$ with $n_k$ the number of
parameters at unit $k$, because otherwise the estimated matrix $A^{(t)}$
will be close to a low-rank matrix, and presumably a poor approximation
of the true matrix $A$, unreliable for numerical inversion.

The same remark applies to the number $n_\mathrm{init}$ of samples used
for initialization: it should be somewhat
larger than the number of parameters at each unit, otherwise $A^{(0)}$
will be of low rank.

\section{Constructing invariant algorithms: Riemannian metrics for neural
networks}
\label{sec:invmet}

\subsection{Gradient descents and metrics, natural metrics}
\label{sec:gradintro}

The gradient of a function $f$
on $\R^d$ gives the direction of steepest ascent: among all (very small)
vectors with a given norm, it provides the greatest variation of $f$.
Formally, the gradient $\grad f$ of a smooth function $f$ is defined by the
property that
\begin{equation}
f(x+\eps v)=f(x)+\eps\langle \grad f,v\rangle +O(\eps^2)
\end{equation}
for any vector $v$, for small enough $\eps$. This depends on the
choice of a scalar product $\langle \cdot,\cdot \rangle$.
In an orthonormal basis, the coordinates of the gradient are simply the
partial derivatives $\partial
f/\partial x_i$ so that gradient descent is
\begin{equation}
x_i\gets x_i-\eta\, \partial f/\partial x_i
\end{equation}
in an orthonormal basis.

For a given norm of the vector $v$, 
the quantity $\langle \grad f,v\rangle$ is maximal when $v$ is collinear
with $\grad f$: so the gradient $\grad f$ indeed gives the direction of
steepest ascent among all vectors with a given norm. The gradient step
$x\gets x-\eta\,\nabla f$ can actually be rewritten (for small enough
$\eta$, up to $O(\eta^2)$ and for regular enough functions $f$) as
\begin{equation}
\label{eq:penmove}
x\gets \argmin_y \left\{ f(y)+\frac{1}{2\eta}\norm{y-x}^2\right\}
\end{equation}
namely, the gradient descent moves into the direction yielding the
smallest values of $f$, penalized by the distance from the current
point\footnote{This can be used to define or study the direction of the
gradient in more general metric spaces (e.g., \cite[Chapter 2]{AGS05}).}.
This makes it clear
how the choice of the scalar product will influence the direction of the
gradient $\grad f$: indeed, another scalar product will define another
norm $\norm{v}^2=\langle v,v\rangle$ for the vector $v$, so that the steepest
direction among all vectors with a given norm will not be the same.
The norm thus defines
directions $v$ in which it is ``cheap'' or ''expensive'' to move; the
gradient is the direction of steepest ascent taking this into account.

If we happen to work in a non-orthonormal basis of vectors
$v_1,\ldots,v_d$, the gradient is given by
$A^{-1} \partial f/\partial x$ where $\partial f/\partial x$ is the
vector of partial derivatives with respect to the coordinates $x_i$ in
the basis, and $A$ is the symmetric matrix made of the
scalar products of the basis vectors with themselves: $A_{ij}\deq
\scal{v_i}{v_j}$. Indeed, the change of variable $\tilde x\deq A^{1/2}x$
provides an orthonormal basis, thus gradient descent for $\tilde x$ is
$\tilde x\gets \tilde x-\eta\,\partial f/\partial \tilde x=\tilde x-\eta
A^{-1/2} \partial f/\partial x$.
Thus, translating back on the variable $x$, the gradient descent of $f$ takes the form
\begin{equation}
\label{eq:grad}
x\gets x-\eta\, A^{-1} \partial f/\partial x
\end{equation}
Conversely, we can start with a norm on $\R^d$ defined through a
positive-definite, symmetric matrix $A$ (which thus defines ``cheap'' and
''expensive'' directions). The gradient descent 
using this norm
will then be given by \eqref{eq:grad}.

So any update of the form \eqref{eq:grad} above with $A$ a symmetric, positive-definite
matrix can be seen as the gradient descent of $f$ using some norm. The
matrix $A$ may even depend on the current point $x$, defining a
\emph{Riemannian metric} in which the norm of a small change $\d x$ of
$x$ is
\begin{equation}
\norm{\d x}^2\deq\sum_{ij} \d x_i A_{ij}(x) \d x_j = \transp{\d x} A(x)\, \d x
\end{equation}

An important feature of gradient descent, in any metric, is that for
$\eta$ small enough, the step is guaranteed to decrease the value of $f$.

The choice of a metric $A$ represents the choice of a norm for vectors in
parameter space. Conversely, choosing a set of parameters and using the
``naive'' gradient ascent for these parameters amounts to implicitly deciding that
these parameters form an orthonormal basis.

For the neural network above,
the gradient ascent $w_{ij}\gets w_{ij}- \eta\frac{\partial
\ell}{\partial w_{ij}}$ corresponds to the choice of $A=\Id$ on parameter
space, namely, the norm of a change of parameters
$\dw=(\dw_{ij})$ is $\norm{\dw}^2\deq \sum
\abs{\dw_{ij}}^2$. Thus, gradient descent for $w_{ij}$ gives the best
$\dw$ for a given norm $\norm{\dw}$, that is, the best change of
$f$ for a given change in the numerical value of the parameters.

\paragraph{Example: from sigmoid to tanh activation function.}
Neural networks using the sigmoid and $\tanh$ activation function are
defined, respectively, by
\begin{equation}
a_k=\sigm(\sum_{i,\,i\to k} a_iw_{ik})
\end{equation}
and
\begin{equation}
a'_k=\tanh(\sum_{i,\,i\to k} a'_i w'_{ik})
\end{equation}
(including the biases $w_{0k}$ and $w'_{0k}$). Since
$\tanh(x)=2\sigm(2x)-1$, the activities of the network correspond to
each other via $a'_k=2a_k-1$ for all $k$ if we set
\begin{equation}
\label{eq:sigmtanh1}
w'_{ik}=\frac{w_{ik}}{4}
\end{equation}
for $i\neq 0$, and
\begin{equation}
\label{eq:sigmtanh2}
w'_{0k}=\frac{w_{0k}}{2}+\frac14 \sum_{i\neq 0} w_{ik}
\end{equation}
for the biases.

Consequently, while the gradient for the sigmoid function will try to
improve performance while minimizing the change to the numerical values
of $w_{ik}$ and $w_{0k}$, the gradient for the tanh function will do
the same for the numerical values of $w'_{ik}$ and $w'_{0k}$, obviously
resulting in different updates. If we follow the $\tanh$ gradient and
rewrite it back in terms of the variables $w_{ik}$, we see that the
$\tanh$ update expressed in the variables $w_{ik}$ is
\begin{align}
w_{ik}\gets w_{ik} +16(\d w_{ik}-\frac12\d w_{0k}) \qquad (i\neq 0)
\end{align}
and
\begin{align}
w_{0k}\gets w_{0k}+4 \d w_{0k}-8\sum_{i\neq 0}(\d w_{ik}-\frac12 \d w_{0k})
\end{align}
where $\d w_{ik}$ is the update that would have been applied to $w_{ik}$
if we were following the standard sigmoid backpropagation. Indeed this
takes the form of a
symmetric matrix applied to $\d w_{ik}$ (the cross-contributions of $\d
w_{0k}$ to $w_{ik}$ and of $\d w_{ik}$ to $w_{0k}$ are the same).

Apart from an obvious speedup factor, an important difference 
between this update and ordinary (sigmoid)
backpropagation on the $w_{ik}$ is that
each time a weight $w_{ik}$ is updated, there is an opposite,
twice as small contribution to $w_{0k}$: in this sense, it is as if this
update assumes that the activities $a_i$ are centered around $1/2$ so
that when $w_{ik}$ gets changed to $w_{ik}+c$, one ``needs'' to add
$-c/2$ to the bias so that things stay the same on average.

\paragraph{Newton's method and gradient descent.} To find the minimum of
a function $f$ on $\R$, one can use the Newton method to solve $f'=0$,
namely, $x\gets x-f'(x)/f''(x)$. In higher dimension this becomes
\begin{equation}
x\gets x-(\Hess f)^{-1} \partial f/\partial x
\end{equation}
where $\partial f/\partial x$ is the vector of partial derivatives, and
$(\Hess f)_{ij}\deq \partial^2 f/\partial x_i \partial x_j$ is the
Hessian matrix of $f$.

Around a non-degenerate minimum of $f$, the Hessian $\Hess f$ will be a
positive-definite matrix. So the Newton method can be seen as a
gradient descent with learning rate $\eta=1$, in the metric $A=\Hess f$,
when one is close enough to a minimum.

\paragraph{Intrinsic metrics.}
There could be a lot of arguing and counter-arguing about the ``right''
way to
write the parameters with respect to which the gradient should be
taken.
The solution to avoid these choices is known: use metrics that depend on
what the system does, rather than on how the parameters are decomposed as
numbers.

The \emph{Fisher metric}, which defines a \emph{natural gradient}
\cite{Amari2000book}, is one
such metric. Namely: the size (norm) of a change of parameters is measured by
the change it induces on the probability distribution of the output of
the model. The symmetric matrix $A$ used in the gradient update is then
the \emph{Fisher information matrix}. We will use scaled-down versions of the Fisher metric for
better scalability.

We present another metric for neural networks, the \emph{backpropagated
metric}. The size of a change of parameters at a given unit is measured
by the effect it has on the units it directly influences, which is itself
measured recursively in the same way up to the output layer. The
matrix defining this metric is obtained by well-known equations related
to the Gauss--Newton approximation of the Hessian.

\subsection{Intrinsic metrics and their computation by backpropagation}
\label{sec:main}

Here we rewrite the definition of neural networks in the language of
differential manifolds and Riemannian geometry; this allows us to define metrics directly in an
intrinsic way.

Consider a neural-like network made of units influencing each other.
The activity of each unit $k$ takes values in a space $\A_k$ which we
assume to be a differentiable manifold
(typically $\R$ without a preferred origin and scale, but we allow room
for multidimensional
activations).
Suppose that the activation of the
network follows
\begin{equation}
\label{eq:genactfunc}
a_k=f^k_{\theta_k}(a_{i_1},\ldots,a_{i_{n_k}})
\end{equation}
where $a_{i_1},\ldots,a_{i_{n_k}}$ are the units pointing to $k$, and where
$f^k_{\theta_k}$ is a function from $\A_{i_1}\times\cdots\times \A_{i_{n_k}}$
to $\A_k$, depending on a parameter $\theta_k$ which itself belongs to a manifold
$\Theta_k$. Here we have no special, always-activated unit coding for
biases: the biases are a part of the parameters $\theta_k$.

We shall also assume that the output units in the network are interpreted
through a final decoding function
to produce an object $\omega=\omega((a_k)_{k\in\Lout})$ relevant to the
initial problem, also assumed to belong to a differentiable manifold.

To implement any gradient ascent over the parameters $\theta$, we first
need a (Riemannian) metric on the parameter space. Such a metric can be
defined by choosing a parametrization by $\R^d$ and deciding that the
elementary vectors of $\R^d$ are orthogonal, but this is not intrinsic:
different parametrizations will lead to different learning trajectories.

In this setting, an object is said to be \emph{intrinsic} if it does not
depend on a choice of parametrization of any of the manifolds involved
(activities, parameters, final output).
Hopefully, casting the activities as elements in an abstract manifold,
and writing intrinsic algorithms that will not depend on how this
manifold is represented as numbers, allows the algorithms to be agnostic
as to any physical interpretation of these activities (activation levels,
activation frequencies, log-frequencies, synchronized activity of a group
of neurons...)

We assume that we are given a meaningful Riemannian metric on the final output
$\omega$: that is, we know how to measure the size of a change in the
output.  For instance, if $\omega$ describes a probability distribution
over a target variable $y$, we can use the Fisher metric over $\omega$
(usually simpler to work with than the Fisher metric over the whole
network parameter space). In the case $\omega$ is a Gaussian of fixed
variance centered on the output values, this coincides with a Euclidean norm
on these values.

Then there are several possibilities to define intrinsic Riemannian metrics on
parameter space. The most direct one is the Fisher metric:
the output is seen as a function of all parameters, and the norm of a
change of parameter $\d\theta$ (over all parameters at once) is the
norm of the change it induces on the output. This is not scalable: for a
neural network with $n$ units, $n\out$ output units, and $d$ incoming edges per unit,
processing each data sample takes $O(n^2d^2+n^2n\out)$, compared to
$O(nd)$ for backpropagation.

A more scalable version is to break down the change of parameter into a
sum of changes of incoming parameters to each unit and take the Fisher
metric at each unit independently.
This is the unitwise Fisher metric. As we will see, it
scales well to sparsely connected networks if the output layer is not too
large: processing each data sample takes $O(nd^2+n\out nd)$.

An even simpler version is the backpropagated metric, defined by
backwards induction from the output layer: the norm of a change of
parameter on the output layer is the norm of the change it induces on the
final result, and the norm of a change of parameter at an internal unit
is the sum of the norm of the resulting changes at the units it
influences directly. Processing each data sample takes $O(nd^2)$.

Quasi-diagonal reduction (Section~\ref{sec:qd}) further produces simplified
intrinsic metrics in which the $O(nd^2)$ terms reduce to $O(nd)$.

\begin{enonce2}{Notation}
In what follows, we use the standard objects of differential geometry but
try to present them in an intuitive way; Appendix~\ref{sec:formal} gives
a fully formal treatment. The notation $\d a$, $\d
\theta$, $\d \omega$ denotes tangent vectors on the corresponding
manifolds (intuitively, differences between two very close values of $a$
or $\theta$ or $\omega$). The notation $\frac{\partial a_i}{\partial
a_k}$ denotes the differential (total derivative) of the activity $a_i$ seen as a function
of $a_k$. In a basis it is represented as the Jacobian matrix of partial
derivatives of the components of $a_i$ w.r.t.\ those of $a_k$ (if
activities are more than $1$-dimensional). In particular, for a numerical function $f$,
$\frac{\partial f}{\partial \theta}$ is represented as a row vector, and
for an infinitesimal change (tangent vector) $\d \theta$ we have $\d
f=\frac{\partial f}{\partial \theta}\d \theta$.
The various metrics involved are $(0,2)$-tensors, but we
use standard matrix notation for them. With this convention a metric
gradient descent on $\theta$ takes the form
$\theta\gets\theta-M(\theta)^{-1}\transp{\frac{\partial f}{\partial \theta}}$.
\end{enonce2}

\begin{defi}[ (Natural metric, unitwise natural metric, backpropagated
metric)]
Let $\norm{\d\omega}^2=\sum
\Omega_{ij}\d\omega_i\d\omega_j=\transp{\d\omega}\Omega\d\omega$ be a metric on the
final output of the network, given by the symmetric, positive-definite
matrix $\Omega$. We define three metrics on the parameter
set.

\begin{itemize}
\item 
The \emph{natural metric} on the parameter set $\theta=(\theta_k)$ is
defined as follows. Let $x$ be an input in the dataset $\D$ and let
$\omega(x)$ be the final output of the network run with input $x$ and
parameter $\theta$. Let $\d\theta$ be a variation of $\theta$ and let
$\d\omega(x)$ be the resulting variation of $\omega(x)$. Let
\begin{equation}
\datnatnorm{\d\theta}{x}^2 \deq \norm{\d\omega(x)}^2
\end{equation}
and then define the natural metric by
\begin{equation}
\natnorm{\d\theta}^2\deq \E_{x\in\D} \datnatnorm{\d\theta}{x}^2
\end{equation}

In matrix form, we have $\d\omega(x)=\frac{\partial \omega(x)}{\partial
\theta}\d\theta$ where $\frac{\partial\omega}{\partial \theta}$ is the
Jacobian matrix of
$\omega(x)$ as a function of $\theta$, so that the natural metric is
given by the matrix
\begin{equation}
\label{eq:natnorm}
\natnorm{\d\theta}^2
=\E_{x\sim \D}\transp{\d\theta}\transp{\frac{\partial\omega(x)}{\partial \theta}}
\Omega\,
\frac{\partial\omega(x)}{\partial \theta}
\d\theta
\end{equation}
The natural metric is given by a matrix of size $\dim\theta=\sum_k \dim
\theta_k$.

\item The \emph{unitwise natural metric} on the parameter set $\theta$ is
\begin{equation}
\unatnorm{\d\theta}^2 \deq \sum_k \natnorm{\d\theta_k}^2
\end{equation}
where $k$ runs over the units of the network and $\d\theta_k$ is the
variation of the incoming parameters to unit $k$.
This metric is given by keeping only the block-diagonal
terms incoming to each unit in the matrix defining the natural metric.

In case $\omega$ is a probability distribution and the metric $\Omega$ on
$\omega$ is the Fisher metric, we also call $\natnorm{\d\theta}$ and
$\unatnorm{\d\theta}$ the \emph{Fisher metric} and \emph{unitwise Fisher
metric}.

\item The \emph{backpropagated metric} on $\theta$ is defined as follows.
Let $x$ be an input in the data. We first define a metric on each of the
activities $a_k$, depending on the input $x$, working from the output
layer backwards.

Given a change $\d a_{k\out}$ in the activity at an output
unit $k\out$, let $\d \omega(x)$ be the corresponding change in the final
output and set
\begin{equation}
\datbpnorm{\d a_{k\out}}{x}^2\deq \norm{\d \omega(x)}^2
\end{equation}
The metric on internal units $k$ is defined as follows: Given a change
$\d a_k$ in the activity of unit $k$, let $\d a_i$ be the resulting
changes in the activities of units $k\to i$ directly influenced by $k$.
Define by induction from the output layer
\begin{equation}
\datbpnorm{\d a_{k}}{x}^2\deq \sum_{i,\,k\to i} \datbpnorm{\d a_i}{x}^2
\end{equation}
Given a change $\d\theta_k$ of the incoming parameters to unit $k$, let
$\d a_{k,x}$ be the change of activity of unit $k$ resulting from the
change $\d\theta_k$, when the network is run on input $x$.
Define the backpropagated metric by
\begin{equation}
\bpnorm{\d\theta_k}^2 \deq \E_{x\in\D}\datbpnorm{\d a_{k,x}}{x}^2
\end{equation}
and
\begin{equation}
\bpnorm{\d\theta}^2\deq \sum_{k\in \L} \bpnorm{\d\theta_k}^2
\end{equation}
\end{itemize}
\end{defi}

Another metric, the \emph{outer product} (OP) metric, can be
defined from slightly different ingredients. It corresponds to an often-used
variant of the natural gradient (e.g., \cite{APF00, TONGA}), in which the expectation under the
current probability distribution is replaced with a similar
term involving only the desired target $y$ for each input $x$
(more details in Section~\ref{sec:usualfish}). 
It can readily be computed by backpropagation from \eqref{eq:op}.

Whereas the metrics above depend on the actual output $\omega(x)$ for
each input $x$, together with a metric on $\omega$, but not on any target
value for $x$, the OP metric depends on a loss function
$\ell(\omega(x),y(x))$ encoding the deviation of $\omega(x)$ from a
desired target $y(x)$ for $x$; but not on a choice of metric for
$\omega$.

\begin{defi}[ (Outer product metric)]
\label{def:op}
For each input $x$ in the dataset $\D$, let $\omega(x)$ be the final
output of the network run with input $x$ and parameter $\theta$. Let
$\ell(\omega(x),y(x))$ be
the loss function measuring how $\omega(x)$ departs from the desired
output $y(x)$ for $x$.

The \emph{outer product metric} is defined as follows. Let
$\d\theta$ be a variation of $\theta$ and let $\d\ell_x$ be the
resulting variation of $\ell(\omega(x),y(x))$. Define
\begin{equation}
\label{eq:defop}
\opnorm{\d\theta}^2\deq \E_{x\in\D} (\d\ell_x)^2
\end{equation}
In matrix form, this metric is
\begin{equation}
\opnorm{\d\theta}^2 = 
\E_{x\sim \D}\transp{\d\theta}\,\transp{\frac{\partial\ell}{\partial \theta}}
\!\frac{\partial\ell}{\partial \theta}
\,\d\theta
\end{equation}
where $\frac{\partial\ell}{\partial \theta}$ is the row vector of
partial derivatives (the differential) of the loss function. Thus this
metric is given by the matrix 
\begin{equation}
\label{eq:op}
\E_{x\sim \D} \transp{\frac{\partial\ell}{\partial
\theta}}\!
\frac{\partial\ell}{\partial \theta}
\end{equation}
hence its name.

The \emph{unitwise outer product metric} is defined by
\begin{equation}
\uopnorm{\d\theta}^2 \deq \sum_k \opnorm{\d\theta_k}^2
\end{equation}
where $k$ runs over the units of the network and $\d\theta_k$ is the
variation of the incoming parameters to unit $k$.
This metric is given by keeping only the block-diagonal
terms incoming to each unit in the matrix defining the tensor-square
differential metric.
\end{defi}

The OP metric has been used simply under the name ``natural
gradient'' in \cite{APF00,TONGA}, which can lead to some confusion because it is distinct from
the natural metric using the true Fisher information matrix (see the
discussion in \cite{BengioNG2013}). Moreover
the OP metric makes sense for optimization situations more
general than the natural gradient, in which the loss function is not
necessarily of the form $\ln p_\theta$ for a probabilistic model $p$. For
these two reasons we
adopt a purely descriptive name\footnote{
The outer product metric is distinct from the ``outer product
(or Levenberg--Marquardt) approximation'' of the Hessian of the loss function
\cite[5.4.2]{Bishop_book}. The latter can be obtained from the natural metric
\eqref{eq:natnorm} in which the output metric $\Omega$ is replaced with
the Hessian of the loss function $\ell_x$ w.r.t.~$\omega$. It depends
on the parametrization of $\omega$. 
For exponential families in the canonical parametrization it coincides
with the Fisher metric.}.

The OP metric is characterized, among all possible metrics, by a unique
property: it provides a
gradient step for which progress is most evenly distributed among all data
samples.

\begin{prop}[ (OP gradient equalizes the gain over the samples)]
\label{prop:equalize}
Let $L\deq \E_{x\in \D} \ell_x$ be the average loss with $\ell_x$ the loss on
input $x$.
The
direction $\d\theta$ given by the gradient of $L$ computed in the outer
product
metric (Def.~\ref{def:grad}) has the following property: Among all
directions $\d\theta$ yielding the same infinitesimal increment $\d L$ at
first order, it is the one for which the increment is most evenly spread
over the data samples $x\in \D$, namely,
$\Var_{x\in\D}\d\ell_x=\E_{x\in\D} (\d\ell_x-\d L)^2$
is minimal.
\end{prop}

The proof is given in the Appendix.  The \emph{unitwise} OP metric does
not, in general, satisfy this property: instead, it minimizes the
variance, over a random data sample $x\in\D$ and a random unit $k$ in the
network, of the contribution to $\d\ell_x$ of the change $\d\theta_k$ at
unit $k$, so that it tries to spread $\d L$ uniformly both over data
samples and units.

\begin{rem}[ (Metric for output parameters)]
\label{rem:outputparams}
The case when the decoding function $\omega((a_k)_{k\in\Lout})$ depends
on additional ``output parameters'' $\vartheta$ (e.g., softmax output with variable
coefficients $\Pr(k)=\e^{\vartheta_k a_k}/(\sum_{k'} \e^{\vartheta_{k'} a_{k'}})$) can be recovered by
considering $\omega$ as an additional output unit to the network, so that
$\vartheta$ becomes the parameter of the activation function of $\omega$.
In particular, applying the definitions above, the metric $\Omega$ on
$\omega$ induces a metric on $\vartheta$
by
\begin{equation}
\bpnorm{\d\vartheta}^2=\natnorm{\d\vartheta}^2=\transp{\d\vartheta}
\left(
\E_{x\in \D}\transp{\frac{\partial \omega}{\partial \vartheta}} \Omega \,\frac{\partial
\omega}{\partial \vartheta}
\right)\d\vartheta
\end{equation}
given by the matrix $\E_{x\in \D}\transp{\frac{\partial \omega}{\partial
\vartheta}} \Omega \,\frac{\partial
\omega}{\partial \vartheta}$.
So $\vartheta$ can be trained by gradient descent in this metric.
\end{rem}

\bigskip

For the Fisher metric, the following is well-known.

\begin{prop}[ (Invariance)]
The natural metric, unitwise natural metric, backpropagated metric, and
plain and unitwise outer product metrics are intrinsic:
$\natnorm{\d\theta}$, $\unatnorm{\d\theta}$, $\bpnorm{\d\theta}$,
$\opnorm{\d\theta}$, and $\uopnorm{\d\theta}$ do not depend on
a choice of parametrization for the activations $a_k$ and for the parameter
$\theta_k$ at each unit $k$.
\end{prop}

\begin{proof}
These metrics have been defined without any reference to
parametrizations, directly by defining the norm of a tangent vector
$\d\theta$; see Appendix~\ref{sec:formal} for a more formal treatment. Consequently the value of the norm $\norm{\d\theta}$ is the same
expressed in any coordinate system \cite[2.1]{GHL87}.
\end{proof}

The natural metric actually has stronger invariance properties than the
unitwise natural metric: it does not depend on a change of
parametrization of the whole parameter $\theta=(\theta_k)$ that would mix
the various components. As such, the unitwise natural metric depends on a
choice of decomposition of the network into units, while the natural
metric is only a function of the input-output relationship of the whole
network. The same holds for the OP and unitwise OP metrics.

\begin{rem}[ (Unitwise metrics as change in activation profile)]
\label{rem:activation}
We saw above that the metric used to define a gradient represents a
``cost'' of moving in certain directions. All three unitwise metrics (unitwise natural, backpropagated, and
unitwise OP) share a common property: 
these metrics decompose as a sum, over the units $k$, of terms
of the form $\norm{\d\theta_k}^2 =
\E_{x\in\D}\,c_{x,k}\norm{\d a_k(x)}^2$ where $\d a_k(x)$ is the
resulting 
change of activity at $k$ on input $x$, and 
$c_{x,k}$ is a weight (different for these three metrics) estimating the
influence of $k$ on the output.
Thus, the ``cost'' of a change at unit $k$
according to these metrics, is an average square norm of the resulting
\emph{change in activation profile $a_k(x)$} over $x$ in the dataset.
This is related to the \emph{best-fit interpretation} of these metrics
(Section~\ref{sec:bestfit}).
\end{rem}

\paragraph{Computing the metrics.} These metrics can be explicitly
computed as follows.

The outer product metric is the easiest to compute:
the terms $\frac{\partial\ell}{\partial \theta}$ are directly computed
by ordinary backpropagation, namely, $\frac{\partial\ell}{\partial
\theta_k}=\frac{\partial \ell}{\partial a_k} \frac{\partial a_k}{\partial
\theta_k}$ where $\frac{\partial \ell}{\partial a_k}$ ($=b_k$) is computed by
backpropagation and $\frac{\partial a_k}{\partial
\theta_k}$ is obtained from the activation function at unit $k$. Then the
matrix defining the metric is $\E_{x\in\D}\transp{\frac{\partial\ell}{\partial
\theta}}\frac{\partial\ell}{\partial
\theta}$, or, for the unitwise version,
$\E_{x\in\D}\transp{\frac{\partial\ell}{\partial
\theta_k}}\frac{\partial\ell}{\partial
\theta_k}$ at each unit $k$.

To compute the natural and unitwise natural metrics, it is enough to compute the Jacobian matrix
$\frac{\partial\omega}{\partial \theta}$. This can be done by performing
one backpropagation for each component of the output layer, for each input $x\in \D$, as follows.

\begin{defi}[ (Backpropagation transfer rates)]
\label{def:trates-gen}
Let $k\out$ be an output unit and let $k$ be any unit in the network. The
\emph{backpropagation transfer rate} $J^{k\out}_k$ from $k\out$ to $k$ is
the $\dim(a_{k\out})\times \dim(a_k)$ matrix defined by
\begin{equation}
\begin{cases}
J^{k\out}_{k\out}\deq\Id_{\dim(a_{k\out})}
\\
J^{k\out}_{k}\deq 0
& \text{for $k\neq k\out$ in the output layer $\Lout$}
\\
J^{k\out}_k \deq \sum_{j,\,k\to j} 
J^{k\out}_j \frac{\partial a_j}{\partial a_k}
& \text{for non-output units $k$}
\end{cases}
\end{equation}
where $\frac{\partial a_j}{\partial a_k}$ is the Jacobian matrix of the
activation function from unit $k$ to unit $j$. Then we have
$J^{k\out}_k=\frac{\partial a_{k\out}}{\partial a_k}$.
\end{defi}

This depends on an input $x$: the activation state of the network has to be
computed by forward propagation before these quantities can be computed.

Typically the activities are one-dimensional, not multidimensional, so that each
$J^{k\out}_{k}$ is just a number, not a matrix.
In this case, all the transfer rates $J^{k\out}_k$ can be computed by performing $n\out$
distinct backpropagations each initialized with a single $1$ on the
output layer.

Since the influence of the parameter $\theta_k$ on the output goes
through the activity of unit $k$,
the unitwise natural metric at $k$ can be computed from
a single number (if activities are one-dimensional)
measuring the influence of unit $k$ on the output, the \emph{Fisher
modulus}.

\begin{defi}[ (Fisher modulus)]
Let $x$ be an input.
Let $k$ be a unit in the network. Let $\Omega$ be the metric on the final
output $\omega$. The \emph{Fisher modulus} $\Phi_k(x)$ of $k$ on input
$x$ is the
$\dim(a_k)\times \dim(a_k)$ matrix given by
\begin{equation}
\label{eq:fishmod-gen}
\Phi_k(x)\deq \transp{\left(\sum_{k\out\in \Lout}
\frac{\partial \omega}{\partial a_{k\out}}J^{k\out}_k
\right)}\Omega \left(\sum_{k\out\in \Lout}
\frac{\partial \omega}{\partial a_{k\out}}J^{k\out}_k
\right)
\end{equation}
For each input $x$, the Fisher modulus is an intrinsic metric on $a_k$: for a given input
$x$, the norm
\begin{equation}\fmnorm{\d a_k}^2\deq \transp{\d a_k} \Phi_k \d
a_k\end{equation} does not
depend on any choice of parametrization.
\end{defi}

Note that $\frac{\partial \omega}{\partial a_{k\out}}$ depends on the
output layer interpretation but not on any parameter $\theta$. Thus,
since the
transfer rates $J$ can be computed by backpropagation, the Fisher
modulus only involves known quantities.

\begin{prop}[ (Computation of the unitwise natural metric)]
The unitwise natural metric at unit $k$ is given by
\begin{align}
\unatnorm{\d\theta_k}^2 &=
\E_{x\in\D}\,\fmnorm{\d a_k(x)}^2
\\&=
\E_{x\in\D}\,\transp{\d\theta_k}\transp{\frac{\partial
a_k}{\partial \theta_k}} \Phi_k \,\frac{\partial
a_k}{\partial \theta_k}\d\theta_k
\end{align}
where $\d a_k(x)$ is the variation of $a_k(x)$ induced by
$\d\theta$, and $\frac{\partial
a_k}{\partial \theta_k}$ is the Jacobian matrix of the activation
function at $k$.
Thus the matrix defining the unitwise natural metric at unit $k$ is
\begin{equation}
F^{(k)}=\E_{x\in\D}\,\transp{\frac{\partial
a_k}{\partial \theta_k}} \Phi_k\, \frac{\partial
a_k}{\partial \theta_k}
\end{equation}
\end{prop}

\begin{proof}
By definition of the transfer rates $J$ we have 
$J^{k\out}_k=\frac{\partial a_{k\out}}{\partial a_k}$. Thus
$\sum_{k\out\in \Lout}
\frac{\partial \omega}{\partial a_{k\out}}J^{k\out}_k=\frac{\partial
\omega}{\partial a_k}$ so that
\begin{equation}
\Phi_k=\transp{\frac{\partial \omega}{\partial
a_k}} \Omega\,\frac{\partial
\omega}{\partial a_k}
\end{equation}
hence
\begin{equation}
\transp{\frac{\partial
a_k}{\partial \theta_k}} \Phi_k \,\frac{\partial
a_k}{\partial \theta_k}
=\transp{\frac{\partial
a_k}{\partial \theta_k}}\transp{\frac{\partial \omega}{\partial
a_k}} \Omega\,\frac{\partial
\omega}{\partial a_k}
\frac{\partial
a_k}{\partial \theta_k}
=\transp{\frac{\partial \omega}{\partial \theta_k}} \Omega\,\frac{\partial
\omega}{\partial \theta_k}
\end{equation}
which, after averaging over the dataset, is the definition of the
unitwise natural metric at $k$.
\end{proof}

An analogous formula can be defined for the full (rather than unitwise)
Fisher matrix, by defining a Fisher modulus $\Phi_{kk'}$ indexed by two
units, and using unit $k'$ on the left and $k$ on the right
in~\eqref{eq:fishmod-gen}. Then the block of entries of the Fisher matrix
corresponding to parameters $\theta_k$ and $\theta_{k'}$ is
\begin{equation}
\E_{x\in\D}\,\transp{\frac{\partial
a_{k'}}{\partial \theta_{k'}}} \Phi_{kk'}\, \frac{\partial
a_k}{\partial \theta_k}
\end{equation}
(see also Proposition~\ref{prop:fishisfish}).

\bigskip

The unitwise Fisher metric is costly to compute when the output layer is
large. We can define another intrinsic metric for the activation of
unit $k$, simply by backpropagating the metric of the output layer.

The changes in the output induced by a change of $\theta_k$ all transit
through the activation of unit $k$. So if we have an
intrinsic metric $\norm{\d a_k}^2$ for the activation of unit $k$, we
can immediately define an intrinsic metric for $\theta_k$, by looking
at the resulting change $\d a_k=\frac{\partial a_k}{\partial
\theta_{k}}\d \theta_{k}$ induced by a change $\d
\theta_k$, and defining the norm of $\d \theta_k$ to be
the norm of the resulting $\d a_k$. If the metric on $a_k$ is given,
in some parametrization of $a_k$, by $\norm{\d a_k}^2=\transp{\d a_k} g_k
\d a_k$ where $g_k$ is a symmetric, positive-definite matrix of size $(\dim
a_k)\times (\dim a_k)$,
then
defining $\norm{\d\theta_k}$ to be the
norm of this $\d a_k$,
\begin{equation}
\norm{\d\theta_k}\deq \norm{\frac{\partial a_k}{\partial
\theta_k}\d\theta_k}
\end{equation}
yields
\begin{equation}
\norm{\d\theta_k}^2=
\transp{\left(\frac{\partial
a_k}{\partial \theta_k}\d\theta_k\right)} g_k \left(\frac{\partial
a_k}{\partial \theta_k}\d\theta_k\right)
\end{equation}
in other words, the matrix defining this metric is $\transp{\frac{\partial
a_k}{\partial \theta_k}} g_k \frac{\partial
a_k}{\partial \theta_k}$.

The unitwise Fisher metric is obtained from the
Fisher modulus by this construction. We now define another intrinsic
modulus playing the same role for the backpropagated metric.

\begin{prop}[ (Backpropagated modulus and computation of the backpropagated metric)]
Let $x$ be an input. Let $k$ be a unit in the network. Let $\Omega$ be the metric on the final
output $\omega$. The \emph{backpropagated modulus} $m_k(x)$ at $k$ is the
$\dim(a_k)\times \dim(a_k)$ matrix given by
\begin{equation}
m_k(x)\deq
\begin{cases}
\transp
{\frac{\partial \omega}{\partial a_k}}\Omega\,
\frac{\partial \omega}{\partial a_k}
&\text{for $k$ in the output layer}
\\
\sum_{j,\,k\to j} \transp
{\frac{\partial a_j}{\partial a_k}}m_j\,
\frac{\partial a_j}{\partial a_k}
&\text{for $k$ an internal unit}
\end{cases}
\end{equation}

Then, for each input $x$, the backpropagated metric on $a_k$ is given by
the backpropagated modulus:
\begin{equation}
\datbpnorm{\d a_k}{x}^2= \transp{\d a_k} m_k \d a_k
\end{equation}
and so the backpropagated metric on 
$\theta_k$ is
given by the matrix $\E_{x\in \D}\,\transp{\frac{\partial
a_k}{\partial \theta_k}} m_k \frac{\partial
a_k}{\partial \theta_k}$, namely,
\begin{equation}
\bpnorm{\d\theta_k}^2=
\E_{x\in \D}\,
\transp{\d\theta_k}\transp{\frac{\partial
a_k}{\partial \theta_k}} m_k \,\frac{\partial
a_k}{\partial \theta_k}\d\theta_k
\end{equation}
\end{prop}

\begin{proof}
Immediate from the definition of the backpropagated metric and $\d a_i=\frac{\partial a_i}{\partial a_k}\d a_k$ and $\d
a_k=\frac{\partial a_k}{\partial \theta_k}\d\theta_k$.
\end{proof}

Like the Fisher modulus, the backpropagated modulus is a single number
when
activities are one-dimensional.
The cost of its computation is the same as one
backpropagation pass.

The equation defining the backpropagated modulus is well-known: it is
related to the so-called Gauss--Newton approximation to the Newton method
(see for instance~\cite{LBOM96}, Section~7), which consists in computing
the Hessian of the loss function and throwing away all terms involving
the second derivative of the activation function (those could result in
non--positive-definite terms, in which case the Newton method is
ill-behaved), with the additional approximation that cross-terms between
different units are also thrown away.  Here we see that no approximation
is involved: we do not throw away annoying terms, we simply define an
intrinsic metric. There is actually no meaningful notion of the Hessian
of a function on manifolds \cite[paragraph 3.37]{GHL87} unless additional structure (affine,
Riemannian) is given or we are at a critical point of the function; the
annoying terms above are precisely the terms that
prevent such a notion to exist. So one could even say, in the context of
differential geometry, that the Newton method is an approximation of the
backpropagated metric rather than the other way round.

The backpropagated modulus and the Fisher modulus are related: If one
tries to write a backpropagated equation to compute the Fisher modulus
$\Phi_k$ in terms of the Fisher modulus at units pointed by $k$, one finds a
quadratic (instead of linear) backpropagation equation with terms
involving pairs of units. Keeping only the terms involving a single
unit yields the equation defining the backpropagated modulus.

\subsection{Quasi-diagonal reduction of a unitwise metric}
\label{sec:qd}

The unitwise Fisher metric, backpropagated metric, and unitwise OP
metric, still involve a full matrix on the incoming parameter space at each unit,
and are thus not adapted if network connectivity is large. We now
introduce two metrics enjoying lesser invariance properties than the
above, but quicker to compute.

Given an intrinsic metric $\norm{\d\theta_k}$ on $\theta_k$ (such as the unitwise
Fisher or backpropagated metric), we are going to define a simpler one,
$\qdnorm{\d\theta_k}$. The inverse of the matrix defining this metric
will be \emph{quasi-diagonal}, with the only non-zero diagonal terms
being those between a weight and the bias.  This will allow for quick
gradient steps costing no more than classical backpropagation.

This relies on the affine structure in neural networks: this
simplification makes sense in a somewhat more restricted setting than the general setting above.
Suppose
that the activation function
\begin{equation}
a_k=f^k_{\theta_k}(a_{i_1},\ldots,a_{i_{n_k}})
\end{equation}
can be written as a composition of a fixed, non-linear activation
function $\phi$ and a quantity $y_k$ that is an affine function of
$a_{i_1},\ldots,a_{i_{n_k}}$:
\begin{equation}
\label{eq:phiaffine}
a_k=\phi(y_{k,\theta_k}(a_{i_1},\ldots,a_{i_{n_k}}))
\end{equation}
such that when $\theta_k$ ranges over its values, $y_{k,\theta_k}$ ranges
over all possible affine functions of $a_{i_1},\ldots,a_{i_{n_k}}$.
For this to make sense, we now have to assume that the activities $a_k$
live in an affine space. So let us go back
to activities with values in
$\R$, but without any preferred origin and scale for activities: we look for invariance under
replacement of $a_i$ with $\alpha_i a_i+\beta_i$ and likewise for $y_i$. 

In any given
parametrization (choice of origin and basis) of $a_i$ and $y_k$ we can write
\begin{equation}
\label{eq:affiney}
y_k=\sum_{i,\,i\to k} w_{ik} a_i+w_{0k}
\end{equation}
for some values $w_{ik}$;
specifying the parameter $\theta_k$ is
equivalent to specifying these quantities.

But this decomposition will change if we change the affine
parametrization of activities: if $a'_i=\alpha_i a_i+\beta_i$ and
$y'_k=\gamma_k y_k+\delta_k$ the relation becomes $y'_k=\delta_k+\sum_i
\gamma_k w_{ik} \alpha_i^{-1}(a'_i-\beta_i)+\gamma_k w_{0k}=\sum_i
(\gamma_k w_{ik}\alpha_i^{-1})a'_i+(\gamma_k w_{0k}+\delta_k-\sum_i
w_{ik}\alpha_i^{-1}\beta_i)$ so that the new weights are
$w'_{ik}=\gamma_k w_{ik}\alpha_i^{-1}$ and the new bias is
$w'_{0k}=\gamma_k w_{0k}+\delta_k-\sum_i w_{ik}\alpha_i^{-1}\beta_i$. In
particular we see that there is no intrinsic ``separation'' between the
bias and the weights; but that there is a separation between $w_{ik}$ and
$w_{i'k}$ for different incoming units $i$ and $i'$. This is formalized
as follows.


Let $\d\theta_k$ be a change of parameter $\theta_k$. For $i\neq 0$, let $\d
w_{ik}$ be the resulting change of $w_{ik}$ in a given parametrization,
and let $\d w'_{ik}$ be the resulting change in another parametrization. If $\d w_{ik}=0$, then we have $\d w'_{ik}=0$ as well in any other
affine
parametrization, since $w'_{ik}=\gamma_k w_{ik}\alpha_i^{-1}$ yields $\d
w'_{ik}=\gamma_k \d w_{ik}\alpha_i^{-1}$. Note that this does not depend
on the input $x$ either\footnote{This relies on the affine form
\eqref{eq:affiney} of $y_k$. If $y_k$ is not affine in
\eqref{eq:phiaffine}, having a constant $\partial y_k/\partial a_i$
is not a well-defined notion and may depend
on the input.}.
Thus,
having $\d w_{ik}=0$ is a property of $\d\theta_k$ that does not depend
on the chosen affine parametrization of activities: it is an intrinsic
property of the change of parameter $\d\theta_k$.
Say that a change of parameter $\d \theta_k$ \emph{does not involve unit
$i$} if $\d w_{ik}$ vanishes.

For the bias the situation is different: the expression $w'_{0k}=\gamma_k w_{0k}+\delta_k-\sum_i w_{ik}\alpha_i^{-1}\beta_i$ giving the bias
in a parametrization from the bias in another parametrization is more
complex, and so the fact that $\d w_{0k}=0$ does depend on the
parametrization. This is where the metric $\norm{\d\theta_k}$ we are trying to simplify
comes into play.

Say that a change of parameter $\d\theta_k$ is \emph{pure bias} if it does not involve any unit
$i$ incoming to $k$, i.e., if $\d w_{ik}=0$ for all $i\neq 0$. This is an
intrinsic condition.  Say that $\d\theta_k$ is \emph{bias-free} if it is
orthogonal, in the metric $\norm{\d\theta_k}$ we are trying to simplify, to all
pure-bias vectors.
Being bias-free is an intrinsic condition, because by
assumption the metric
$\norm{\d\theta_k}$ is intrinsic. Being bias-free does not simply mean $\d
w_{0k}=0$; let us work it out in
coordinates.

Let
$A_{ii'}$ be the symmetric matrix defining the metric $\norm{\d\theta_k}$ in
a given parametrization. The associated scalar product is
\begin{equation}
\langle \d\theta_k,\d\theta'_k\rangle=
\sum_i\sum_{i'} 
A_{ii'} \d w_{ik} \d w'_{i'k}
+\sum_i A_{0i}(\d w_{0k}\d w'_{ik} + \d w'_{0k}\d w_{ik})
+A_{00}\d w_{0k}\d w'_{0k}
\end{equation}
with $A_{0i}=A_{i0}$.

In particular, if the only non-zero component of
$\d\theta_k$ is $\d w_{ik}$, then its scalar product with a pure bias
$\d w'_{0k}$ will be $A_{0i}\d w'_{0k}\d w_{ik}$. On the other hand,
if to $\d\theta_k$ we add a bias component
$\d w_{0k}=-A_{00}^{-1}A_{0i}\d w_{ik}$, then the scalar product with any
pure bias will vanish. Such a $\d\theta_k$ is thus bias-free.

In the case when the parameter $\theta_k$ allows to represent all affine
functions of the incoming activations, we can decompose a
variation $\d\theta_k$ of $\theta_k$ into components $\d\theta_{ki}$ each
involving only one incoming unit $i$, and a pure bias component
$\d\theta_{k0}$. This
decomposition is unique if we impose that each $\d\theta_{ki}$ is
bias-free. Explicitly, if in some parametrization we have $\d\theta_k=(\d
w_{0k},\d w_{1k},\ldots,\d w_{n_kk})$ this decomposition is
\begin{equation}
\d\theta_{ki}=(-A_{00}^{-1}A_{0i}\d w_{ik},0,\ldots,\d w_{ik},0,\ldots,0)
\end{equation}
and 
\begin{equation}
\d\theta_{k0}=(\d w_{0k}+\sum_i A_{00}^{-1}A_{0i}\d w_{ik},0,\ldots,0)
\end{equation}
The decomposition $\d\theta_k=\d\theta_{k0}+\sum_i \d\theta_{ki}$ is intrinsic.

We can then define a new intrinsic metric on $\d\theta_k$ by setting
\begin{equation}
\qdnorm{\d\theta_k}^2\deq \norm{\d\theta_{k0}}^2+\sum_i
\norm{\d\theta_{ki}}^2
\end{equation}
which is readily computed:
\begin{equation}
\begin{aligned}
\qdnorm{\d\theta_k}^2&=
A_{00} \left(\d w_{0k}+\sum_i A_{00}^{-1} A_{0i} \d w_{ik}\right)^2
\\&\qquad +\sum_i (A_{ii}\d w_{ik}^2-2g_{0i} (A_{00}^{-1}A_{0i} \d w_{ik})\d
w_{ik}
+A_{00}(A_{00}^{-1}A_{0i} \d w_{ik})^2)
\\&=
A_{00}\d w_{0k}^2+2\sum_i A_{0i} \d w_{0k}\d w_{ik}
+\sum_{i,i'} A_{00}^{-1}A_{0i}A_{0i'} \d w_{ik} \d w_{i'k}
\\&\qquad+\sum_i (A_{ii}-A_{00}^{-1}A_{0i}^2)\d w_{ik}^2
\end{aligned}
\end{equation}

Thus, this metric is defined by a matrix $\tilde A$ given by
$\tilde A_{00}=A_{00}$, $\tilde A_{0i}= A_{0i}$ and
$\tilde A_{ii'}=A_{00}^{-1}
A_{0i}A_{0i'}+\1_{i=i'}(A_{ii}-A_{00}^{-1}A_{0i}^2)$.

\begin{defi}
\emph{Quasi-diagonal reduction} is the process which, to an
intrinsic metric defined by a matrix $A$ in affine coordinates,
associates the metric defined by the
matrix
\begin{equation}
\tilde A\deq \diag(A)+A_{00}^{-1}(v\otimes
v)-\diag(A_{00}^{-1}(v\otimes v))
\end{equation}
where
\begin{equation}
v=(A_{00},\ldots,A_{0i},\ldots)
\end{equation}

The \emph{quasi-diagonal backpropagated metric} is the quasi-diagonal
metric obtained from the backpropagated metric. The \emph{quasi-diagonal
Fisher metric} is the one obtained from the unitwise Fisher metric. The
\emph{quasi-diagonal OP metric} is the one obtained from the unitwise
OP metric.
\end{defi}

The reasoning in this section can be summarized as follows.

\begin{prop}
Assume that the activation function is a fixed non-linear function
composed with an affine function.
Then the quasi-diagonal reduction $\tilde A$ of an intrinsic metric $A$ is
intrinsic.
\end{prop}

Importantly, the matrix
$\tilde A=\diag(A)+A_{00}^{-1}(v\otimes
v)-\diag(A_{00}^{-1}(v\otimes v))$
is the sum of a diagonal matrix and
a rank-$1$ matrix.
This allows for easy inversion, resulting in a
quasi-diagonal inverse matrix.

\begin{prop}[ (Quasi-diagonal gradient step)]
Let $\tilde A$ be the quasi-diagonal reduction of $A$.
Let $b=(b_0,\ldots,b_i,\ldots)$ and $w=\tilde A^{-1}b$. Then $w$ is given
by
\begin{align}
w_i&=\frac{b_iA_{00}-b_0A_{0i}}{A_{ii}A_{00}-A_{0i}^2}
\qquad\text{ for $i\neq 0$}
\\w_0&=\frac{b_0}{A_{00}}-\sum_{i\neq 0}\frac{A_{0i}}{A_{00}}w_i
\end{align}
\end{prop}

Thus, only the entries $A_{00}$, $A_{ii}$ and $A_{0i}$ of the original
matrix $A$ need to be known in order to implement gradient descent using
the quasi-diagonal metric defined by $\tilde A$.

Note that if the original matrix $A$ is positive-definite, we have
$A_{00}>0$ and
$A_{00}A_{ii}>A_{0i}^2$ (by the Cauchy--Schwarz inequality applied to the
first and $i$-th basis vectors), so that the solution $w$ above is
well-defined and unique.

\subsection{Intrinsic gradients}
\label{sec:grads}

Thanks to these intrinsic metrics we can define intrinsic gradient directions in
parameter space. Given a dataset $\D$ of inputs $x$ and corresponding
targets $y$, the average loss function is 
\begin{equation}
L_\theta\deq \E_{x\in \D}\ell_\theta(y)
\end{equation}
where we put a subscript $\theta$ to make explicit its dependency on the
parameters of the network.
Given an intrinsic metric $\norm{\cdot}$, the differential
\begin{equation}
G=-\frac{\partial L_{\theta}}{\partial \theta}
\end{equation}
of the average loss
with respect to the full parameter set $\theta$, defines a
gradient direction $\grad_\theta L$ by the usual definition: it is the
only tangent vector such that for any
$\d\theta$ we have
\begin{equation}
L_{\theta+\d\theta}=L_{\theta}+\langle \grad_\theta
L,\d\theta\rangle+O(\norm{\d\theta}^2)
\end{equation}
where $\langle \cdot,\cdot\rangle$ is the scalar product associated with
the norm $\norm{\cdot}$. In a parametrization where $\norm{\cdot}^2$ is
given by a symmetric, positive definite matrix $A$, the gradient is given
by
\begin{equation}
\grad_\theta L=A^{-1}\frac{\partial L}{\partial \theta}=-A^{-1}G
\end{equation}
The gradient $\grad_\theta L$ is an intrinsic tangent vector on the
parameter set.

\begin{defi}
\label{def:grad}
The natural gradient, unitwise natural gradient, backpropagated metric gradient, 
OP gradient, unitwise OP gradient, and their quasi-diagonal reductions,
respectively, are the following update rule for $\theta$:
\begin{equation}
\theta\gets \theta-\eta \grad_\theta L
\end{equation}
where $\grad_\theta L=A^{-1}\frac{\partial L}{\partial \theta}$ is the gradient of the average loss function $L$
computed with $A$ the natural metric, unitwise natural metric, backpropagated
metric, OP metric, unitwise OP metric, and their
quasi-diagonal reductions, respectively.
\end{defi}

The algorithms of Section~\ref{sec:algos} are the application of these
updates to ordinary neural networks, written out with $[0;1]$-valued
activities and sigmoid activation function. More details on how this
works out are given below (Section~\ref{sec:usualfish}).

This update is intrinsic only under all \emph{affine} reparametrizations of
the parameter $\theta$. Indeed, even if the tangent vector
$\grad_\theta L$ giving the direction of the gradient is fully
intrinsic, adding a tangent vector to a parameter $\theta$ is not an
intrinsic operation (if two parametrizations differ by a non-affine
transformation, then the additions will not amount to the same).

On the
other hand, the ideal limit when the learning rate $\eta$ tends to $0$ is
intrinsic: the trajectories of the differential equation
\begin{equation}
\label{eq:gradtraj}
\frac{\mathrm{d}\theta(t)}{\mathrm{d}t}=-\grad_{\theta(t)} L
\end{equation}
are intrinsic trajectories in parameter space for the unitwise natural
gradient and backpropagated metric.

For the quasi-diagonal algorithms, invariance is always restricted to
affine reparametrizations, since this is the setup in which they are
well-defined.

\subsection{The Fisher matrix for neural networks}
\label{sec:usualfish}

The general definitions above depend on the choice of a metric on the
output $\omega$ of the network. When this metric is the Fisher metric on the
output layer, applying the general definitions above to ordinary neural
networks leads to the algorithms of Section~\ref{sec:algos}.  This is
mostly by direct computation and we do not reproduce it fully. Let us
however discuss the Fisher metric in more detail.

For each input $x$, the network defines a probability distribution
$\omega$ on the
outputs $y$. This
probability distribution depends on the parameters of the network. Thus,
for each input $x$, we can define a \emph{datum-wise} Fisher matrix on
the parameter set:
\begin{equation}
\label{eq:inputfish}
F(x)_{w_{ij}w_{i'j'}}=\E_{y|x} \frac{\partial \ell(y)}{\partial
w_{ij}}\frac{\partial \ell(y)}{\partial w_{i'j'}}
\end{equation}
where as above $\ell(y)=\ln \omega(y)$ and where $\E_{y|x}$ denotes
expectation for $y$ following the distribution $\omega$ defined by the
input $x$.

The dataset together with the network define a probability distribution
on pairs $(x,y)$, by first choosing at random an input $x$ in the
dataset, then running the network on this input. The Fisher matrix
associated with this distribution on pairs $(x,y)$ is the average of
the datum-wise Fisher matrix over the dataset
\begin{equation}
F=\E_{x\in \D} F(x)
\end{equation}
(see \cite{Amari2000book}, Section~8.2), or more explicitly
\begin{equation}
F_{w_{ij}w_{i'j'}}=\E_{x\in\D} \E_{y|x} \frac{\partial \ell(y)}{\partial
w_{ij}}\frac{\partial \ell(y)}{\partial w_{i'j'}}
\end{equation}

\paragraph{Exact Fisher matrix versus one-sample Fisher matrix.} One
possible way to train neural networks using the natural gradient is to
estimate the Fisher matrix by taking an input $x$ in the dataset,
taking a random output $y$ for this input, and add the term
$\frac{\partial \ell(y|x)}{\partial w_{ij}}\frac{\partial
\ell(y|x)}{\partial w_{i'j'}}$ to the current estimate of the Fisher
matrix (with a discount factor for older contributions in an online
setting). This leads to the \emph{Monte Carlo natural gradient} with $K$
samples $y$ per input $x$:
\begin{equation}
\label{eq:MCnatgrad}
\hat F_{w_{ij}w_{i'j'}}=\E_{x\in \D}\,\frac{1}{K}\sum_{k=1}^{K}
\frac{\partial \ell(y_k)}{\partial
w_{ij}}\frac{\partial \ell(y_k)}{\partial w_{i'j'}}
\end{equation}
where each $y_k$ is drawn according to the output probability distribution
$\omega$ defined by the output for input $x$. Even $K=1$ can lead to reasonable
results (Section~\ref{sec:exp}).

An important variant uses for $y$ the actual target value for input $x$, instead
of taking $y$ as a random sample given by the activations of the output
layer:
\begin{equation}
\label{eq:onesamplefisher}
\hat F_{w_{ij}w_{i'j'}}=\E_{x\in \D}  \frac{\partial \ell(y(x))}{\partial
w_{ij}}\frac{\partial \ell(y(x))}{\partial w_{i'j'}}
\end{equation}
with $y(x)$ the target for $x$: this is just the \emph{outer product (OP)
metric} of Definition~\ref{def:op}. It is not an unbiased estimate of the
Fisher metric; still, hopefully
when the network converges towards the desired
targets, the law of $y$ taken from the output distribution converges to
the actual target and the two variants get close. This variant has been present for a
long time in
studies on natural gradient (as is clear, e.g., from
Equation~(14) in \cite{APF00}) and is elaborated upon in
\cite{TONGA}. As pointed out in \cite{BengioNG2013}, the two variants are
often confused.

These two variants are both intrinsic. The
OP variant, contrary to the true natural gradient, depends on the targets
and not only on the network and inputs.

Both the true natural gradient, its Monte Carlo approximation, and its
``one-sample''/OP variant give rise to a unitwise version and to a
quasi-diagonal version (Section~\ref{sec:qd}). For a network with $n$
units and at most $d$ incoming connections per unit, the algorithmic cost
of processing each data sample is $O(Kn^2d^2)$ for the Monte Carlo
natural gradient, $O(Knd^2)$ for its unitwise version and $O(Knd)$ for
its quasi-diagonal reduction. Algorithmic cost for the OP metric is the same
with $K=1$.

In Section~\ref{sec:exp} we compare performance of the unitwise natural
gradient, Monte Carlo unitwise natural gradient with $K=1$, and
unitwise OP natural gradient. We will see that although the OP metric and
the one-sample ($K=1$) Monte Carlo natural gradient  look
similar, the latter can perform substantially better.

\paragraph{Exact Fisher matrix computation.} It is possible
compute the exact Fisher matrix (rather than using a single value for
$y$)
by using the Fisher modulus and backpropagation transfer rates. The
latter can be computed
by doing $n\out$ backpropagations for each input.  This is of course more
convenient than computing the expectation $\E_{y|x}$ by summing over the (in the Bernoulli case) $2^{n\out}$
possible outcomes for $y$.

The backpropagation transfer rates $J_k^{k\out}$ from
Definition~\ref{def:trates} simply implement the general
Definition~\ref{def:trates-gen} for ordinary neural networks. In
Section~\ref{sec:algos}, the unitwise natural gradient was obtained from
these transfer rates through the Fisher modulus. Here we reproduce the
corresponding formula for 
all terms of the Fisher matrix, not only the terms of the unitwise
Fisher matrix incoming to a given unit, so we introduce a Fisher modulus
indexed by pairs of units.

\begin{prop}[ (Exact Fisher matrix for neural networks)]
\label{prop:fishisfish}
Let $x$ be an input for the network. Compute the transfer rates
$J_k^{k\out}$ as in Definition~\ref{def:trates}.
Depending on output layer interpretation, set for each pair of units $k$
and $k'$:
\begin{equation}
\label{eq:fishmod}
\begin{cases}
\Phi_{kk'}(x)\deq \sum_{k\out\in \Lout} J_k^{k\out}J_{k'}^{k\out}
&\text{(square-loss)}
\\
\Phi_{kk'}(x)\deq \sum_{k\out\in \Lout}
\frac{J_k^{k\out}J_{k'}^{k\out}}{a_{k\out}(1-a_{k\out})}
&\text{(Bernoulli)}
\\
\Phi_{kk'}(x)\deq
\frac{1}{S}\sum_{k\out\in\Lout}\e^{a_{k\out}}J_k^{k\out}J_{k'}^{k\out}
&\text{(softmax)}
\\
\phantom{\Phi_{kk'}(x)\deq}\;
-\frac{1}{S^2}
\left(\sum_{k\out\in\Lout} \e^{a_{k\out}}J_k^{k\out}\right)
\left(\sum_{k\out\in\Lout} \e^{a_{k\out}}J_{k'}^{k\out}\right)
\\
\phantom{\Phi_{kk'}(x)\deq}\text{where }S\deq\sum_{k\out\in\Lout}
\e^{a_{k\out}}
\\
\Phi_{kk'}(x)\deq
\frac{4}{S}\sum_{k\out\in\Lout}J_k^{k\out}J_{k'}^{k\out}
&\text{(spherical)}
\\
\phantom{\Phi_{kk'}(x)\deq}\;
-\frac{4}{S^2}
\left(\sum_{k\out\in\Lout} a_{k\out}J_k^{k\out}\right)
\left(\sum_{k\out\in\Lout} a_{k\out}J_{k'}^{k\out}\right)
\\
\phantom{\Phi_{kk'}(x)\deq}\text{where }S\deq\sum_{k\out\in\Lout}
a_{k\out}^2
\end{cases}
\end{equation}

Then the entry of the datum-wise Fisher matrix $F(x)$ associated with
parameters $w_{ik}$ and $w_{jk'}$ (including biases with $i=0$ or
$j=0$) is
\begin{equation}
\label{eq:exactfisherx}
F(x)_{w_{ik}w_{jk'}}=a_ia_j \deractf_k \deractf_{k'} \Phi_{kk'}
\end{equation}
and thus the corresponding entry in the Fisher matrix is
\begin{equation}
\label{eq:exactfisher}
F_{w_{ik}w_{jk'}}=\E_{x\in \D}\,a_ia_j \deractf_k \deractf_{k'}\Phi_{kk'}
\end{equation}
\end{prop}

The proof is given in the Appendix and is a more or less straightforward application
of the results of the previous section, together with an explicit
computation of the Fisher metric on the output in the Bernoulli,
square-loss, and classification interpretations.


So it is possible to compute the full Fisher matrix by
performing $n\out$ independent backpropagations for each sample input.
The Fisher matrix $F$, being the average of $F(x)$ over the dataset,
may
be approximated by the standard online or small-batch techniques
using samples from the dataset.

For a network with only one hidden layer, this simplifies and no
additional backpropagations are needed.
Indeed, the backpropagation transfer rates of the input layer are never
used, and on the hidden layer are given
by
\begin{equation}
J_k^{k\out}=w_{kk\out} \deractf_{k\out}
\end{equation}
from which the Fisher modulus can be immediately computed. This is the
case treated in \cite{Kur94} (for the Bernoulli interpretation). 

\section{Some properties of unitwise algorithms and their
quasi-diagonal approximations}
\label{sec:someprops}

\subsection{Performance improvement at each step}

A common feature of all gradient-based algorithms in any metric is that
the objective function improves at each step provided the learning rate
is small enough. Consequently this holds for the unitwise natural
gradient, backpropagated metric gradient, and their quasi-diagonal
reductions.

\begin{prop}
Suppose that training has not reached a local optimum, i.e., that the
gradient vector $G$ of Section~\ref{sec:algos} does not vanish. Suppose
that the metric considered is non-degenerate (i.e., respectively, that the matrices
$F^{(k)}$ or $M^{(k)}$ are invertible, or that the denominators in the
quasi-diagonal algorithms do not vanish), so that the algorithms
considered are well-defined.
Suppose that the chosen interpretation $\omega$ of the output layer
depends smoothly on the output layer activities.

Then there exists a value $\eta_C$ of the learning rate such that, for
any learning rate $\eta<\eta_C$, the value of the loss function strictly
decreases after one step of the unitwise natural gradient, backpropagated
metric gradient, or their quasi-diagonal reductions.
\end{prop}

As usual, the value of $\eta_C$ depends on the current state of the
network and thus may change over the course of training.

\subsection{Invariance properties}
\label{sec:inv}

The algorithms presented in Section~\ref{sec:algos} are the
implementation of the gradients and metrics defined in
Section~\ref{sec:invmet}, written out using $[0;1]$-valued activities and the
logistic activation function. We could have written them out, for instance, using
$[-1;1]$-valued activities and the $\tanh$ activation function, and the
learning trajectory would be the same—provided, of course, that the
initialization was done so that both implementations of the network
behave
the same at startup. We present a more precise formulation of this
property.

Imagine that the inputs of the network are subjected to simple
transformations such as scaling ($a_i\gets \alpha_i a_i$ for $i$ in the
input layer) or $0$/$1$ inversion ($a_i \gets 1-a_i$). There is a simple
way to change the parameters of subsequent units so that the final
activation of the network stays the same, namely, $w_{ij}\gets
w_{ij}/\alpha_i$ for scaling and $w_{ij}\gets -w_{ij}$, $w_{0j}\gets
w_{0j}+w_{ij}$ for $0$/$1$ inversion. So clearly the expressivity of a
neural network is not sensitive to such changes.

However, training will behave differently. For instance, if we apply one
step of backpropagation training to the scaled inputs with the scaled
network, the coefficients of units which have been scaled down
$(\alpha_i<1)$ will evolve more slowly and conversely for $\alpha_i>1$.
The final output of the network after the update will be different.
(Hence the common practice of rescaling the activities of units.) The
same goes for $0$/$1$ inversion in a slightly more complicated way: evolution of
the bias depends on the activity of input units, and the weights from
input units with activity close to $0$ will evolve faster than those with
activity close to $1$, as seen on Figure~\ref{fig:bptraj}.

We would like the following invariance for a training procedure: If we
start with two networks $N$ and $N'$ which are fed inputs $x$ and $x'$
with $x'$ obtained from a simple transformation of $x$, and if the
parameters of $N'$ are set such that initially its output is the same as
$N$, then we would like the outputs of $N$ and $N'$ to stay the same
after one step of the training procedure.

This is \emph{not} satisfied by backpropagation. However, for any affine
transform of the activities of any unit, this is satisfied by the natural
gradient, unitwise natural gradient, unitwise outer product
gradient, backpropagated metric gradient, and their quasi-diagonal
reductions.

\bigskip

The sigmoid and $\tanh$ networks correspond to each other by the
following rewriting, thanks to $\tanh(x)=2\sigm(2x)-1$: if $a_k=\sigm(\sum_{i\to k} w_{ik}a_i+w_{0k})$ and
$a'_k=\tanh(\sum_{i\to k} w'_{ik}a'_i+w'_{0k})$ (and interpretation of
the output layer in the $\tanh$ case is done by putting back the
activities in $[0;1]$ via $a'\mapsto 1/2+a'/2$), then the two networks
will behave the same if we set $w_{ik}=4w'_{ik}$ ($i\neq 0$) and
$w_{0k}=2w'_{0k}-2\sum_{i\neq 0} w'_{ik}$.

\begin{defi}
Let $k$ be an input or
internal unit.
Call \emph{$(\alpha,\beta,\gamma)$-affine reparametrization of unit $k$} the
following operation: Replace the activation of unit $k$
\begin{equation}
a_k=f^k_{\theta_k}(a_{i_1},\ldots,a_{i_{n_k}})
\end{equation}
where $\theta_k=(w_{0k},(w_{ik})_{i\to k})$,
with
\begin{equation}
a'_k=\alpha f^k_{\gamma\,\theta'_k}(a_{i_1},\ldots,a_{i_{n_k}})+\beta
\end{equation}
where $\theta'_k=(w'_{0k},(w'_{ik})_{i\to k})$.
Send $a'_k$ instead of $a_k$ to the next layer of the network, with
weights modified as follows:
\begin{equation}
w'_{kj}=w_{kj}/\alpha,\qquad w'_{0j}=w_{0j}-w_{kj}\beta/\alpha
\end{equation}
for all units $j$ such that $k\to j$, and $w'_{ik}=w_{ik}/\gamma$ for all
units $i$ with $i\to k$ (including $i=0$), so that the final outputs before and
after the reparametrization are the same.
\end{defi}

The passage from $\sigm$ to $\tanh$ consists in applying the
$(2,-1,2)$-reparametrization to all units.
We have restricted the
definition to non-output units to simplify notation; for output units a
corresponding reparametrization of the output interpretation has to be
done.

The following result is an immediate consequence of the intrinsic definition of
the algorithms. It is only part of the invariance properties of the objects
from Section~\ref{sec:invmet}. In particular, in the limit of small
learning rates ($\eta\to 0$), the trajectories \eqref{eq:gradtraj} of the
unitwise natural gradient, backpropagated
metric gradient, and unitwise OP gradient, are invariant under all smooth
one-to-one reparametrization and not only affine ones.

\begin{prop}[ (Invariance under affine reparametrization of activities)]
Consider a network obtained from an initial network by applying any
number of $(\alpha,\beta,\gamma)$-affine reparametrizations to any number
of units (where $\alpha$, $\beta$ and $\gamma$ may depend on the unit).

Then, after one step of the unitwise natural gradient, backpropagated
metric gradient, Monte Carlo unitwise natural gradient, unitwise OP gradient, or their quasi-diagonal reductions, the final outputs of the
non-reparametrized and reparametrized networks are the same.

Consequently, the learning trajectories, and performance, of the two networks with these
corresponding initializations are the same.
\end{prop}

This may seem a simple thing, but we should keep in mind that this
property is \emph{not} satisfied by backpropagation, or by quasi-Newton
methods if the latter use diagonal approximations of the Hessian.

In particular, the algorithms presented here are insensitive to shifting
and scaling of all units in the network. Traditionally, it is recommended
to normalize the activities on input units so that they average to $0$
over the dataset and have a prescribed variance: the algorithms here
automatically do the same in an implicit way, for all (not only input)
units. As a consequence, units with low activation levels get updated as
fast as highly activated units.  (Note that as discussed after the
definition of the quasi-diagonal algorithms, these averages and variances
are computed according to non-uniform weights on the dataset given by the
Fisher modulus or backpropagated modulus.)

Still the invariance above only applies if the two networks considered
have corresponding initializations. For instance, if the initial weights
are random with a variance set to $1$ whatever the data, obviously the
initial behavior of the network will be sensitive to scaling of its
input. So \emph{these methods do not remove the need for traditional
recommendations for initializing the weights} (either by normalizing the
data and then taking weights of size $1$, or by taking initial weights
depending on the variance or covariance matrix of the data).

\bigskip

The unitwise gradients (natural, backpropagated metric, OP, and Monte
Carlo natural),
but not
their quasi-diagonal reductions, have a further, more interesting invariance
property: invariance under affine recombination of the signals a unit
receives from its various incoming units. For instance, if we start with
zero weights, an internal unit will evolve in the same way if it receives
$f$ and $f+\eps g$ (where $f$ and $g$ are seen as function of the input
$x$) as if it receives $f$ and $g$. This is especially useful if $g$ is
correlated to the desired output.

\begin{prop}[ (Invariance under affine recombination of incoming
signals)]
\label{prop:invrecomb}
Consider a neural network and define a new one in the following way. Let
$k$ be a non-input unit in the network, with $n_k$ incoming units, and
let $\phi:\R^k\to\R^k$ be an invertible affine map. Define a new network
by replacing the activation function at unit $k$
\begin{equation}
a_k=f^k_{\theta_k}(a_{i_1},\ldots,a_{i_{n_k}})
\end{equation}
with
\begin{equation}
a_k=f^k_{\phi^\ast(\theta_k)}(\phi(a_{i_1},\ldots,a_{i_{n_k}}))
\end{equation}
still parametrized by $\theta_k$,
where $\phi^\ast(\theta_k)$ results from applying the
dual\footnote{$\theta_k$ is an affine form over the $n_k$-tuple of
incoming activities. $\phi^\ast$ is defined, axiomatically, by the property that
applying $\phi^\ast(\theta_k)$ to the activities transformed by $\phi$, is
the same as applying $\theta_k$ to the untransformed activities.
Decomposing $\theta_k=(w_{0k},(w_{ik})_{i\to k})$, the affine
matrix defining $\phi^\ast$ is the transpose of the inverse of the affine
matrix defining $\phi$.} inverse affine
transformation $\phi^\ast$ to $\theta_k$, so that initially the responses
of the original and reparametrized networks are the same.

Then, after one step of the unitwise natural gradient, backpropagated
metric gradient, Monte Carlo natural gradient, or unitwise OP gradient, with respect to $\theta_k$, the final outputs of the
non-reparametrized and reparametrized networks are the same.

Consequently, the learning trajectories, and performance, of the two
networks with these corresponding initializations are the same.
\end{prop}

Once more, this is not simply $\phi$ in one place cancelling out
$\phi^{-1}$ in another: indeed, backpropagation or quasi-Hessian methods
do not have this property, and neither do the quasi-diagonally-reduced
algorithms.

\begin{proof}
This comes as a consequence of the best-fit interpretation
(Proposition~\ref{prop:bestfit}) below.

It also follows from the
intrinsic constructions by noting that, unlike the quasi-diagonal
reductions, the construction of these gradients never breaks down the
$n_k$-tuple of incoming activities into its components from each incoming
unit; thus, contrary to the quasi-diagonal reductions, we could have
written the unitwise natural and backpropagated metrics in a setting
where
activation functions are given by
$a_k=f^k_{\theta_k}(g(a_{i_1},\ldots,a_{i_{n_k}}))$ where $g$ is a fixed,
parameterless map with values in a manifold.
\end{proof}


\subsection{Best-fit interpretation}
\label{sec:bestfit}

The unitwise natural gradient, backpropagated metric gradient, and
unitwise OP gradient (but
not their quasi-diagonal reductions) share an interpretation as a
least-squares regression problem at each unit. Namely, the backpropagated
value $b_k(x,y)$ on input $x$ and target $y$ indicates how the activity
of unit $k$ should change on input $x$. Seeing $b_k$ as a function of the
input $x$, unit $k$ has to use the activities of incoming units $i$ (also
seen as functions of $x$) and combine them using the weights $w_{ik}$, to
match $b_k(x,y)$ as close as possible for each $x$.
This idea is presented in \cite{Kur94} in a more specific setting. This is
also relevant to the behavior of the algorithms when the matrices $F$ and
$M$ defining the metrics are singular or close to singular, as we will
see.

\begin{prop}[ (Intrinsic gradients as best fit to $b$)]
\label{prop:bestfit}
Let $k$ be a non-input unit in the network.
For $x$ in the dataset $\D$, let $b_k(x)$ be the backpropagated value
\eqref{eq:b}
obtained on input $x$ and the corresponding target $y$.

Consider the solution $\lambda=(\lambda_i)$ to the following
weighted least-squares problem:
\begin{equation}
\label{eq:leastsquares}
\lambda=\argmin \left\{
\sum_{x\in\D} \left(\sum_i \lambda_i
a_i(x)-\frac{b_k(x)}{\deractf_k(x)\Phi_k(x)}\right)^2
W_x
\right\}
\end{equation}
where $i$ runs over the incoming units to $k$ (including $i=0$ with
$a_0\equiv 1$), 
$\Phi_k(x)$ is the Fisher modulus (Definition~\ref{def:fishmod}),
and the weights are
\begin{equation}
\label{eq:bestfitweights}
W_x\deq \deractf_k(x)^2\Phi_k(x)
\end{equation}

Then the unitwise natural gradient step \eqref{eq:locnatstep} is given by
$\lambda$, namely, the update is $w_{ik}\gets w_{ik}+\eta \lambda_i$ at
each unit $k$.

The same holds for the backpropagated metric gradient using the
backpropagated modulus $m_k$ (Definition~\ref{def:bpmod}) instead of the Fisher
modulus $\Phi_k$.

The same holds for the unitwise OP gradient using $b_k(x)^2$ instead of the
Fisher modulus $\Phi_k$.
\end{prop}

Thus, the gradient step depends on the \emph{linear span} of the incoming
activities $(a_i(x))_{i\to k}$, seen as functions over the dataset
(which, by the way, proves Proposition~\ref{prop:invrecomb} above). This
is why the gradient step is the same whether the unit receives signals
$f(x)$ and $g(x)$ or $f(x)$ and $f(x)+\eps g(x)$. 
Thus, these algorithms perform an implicit orthonormalization of the
incoming signals at each unit (not only input units) in the network.

\begin{proof}
A direct application of the well-known formula for the solution of the
weighted least-squares problem~\eqref{eq:leastsquares}, with the choice
of weight~\eqref{eq:bestfitweights}, yields exactly the updates
\eqref{eq:locnatgrad} and \eqref{eq:bpmgrad}.
\end{proof}

\paragraph{Non-invertibility and regularization of the matrices.} In
several situations the matrices $F$ and $M$ used to define the unitwise natural
update and backpropagated metric update can be singular.

This is the case, for
instance, if one input unit is uniformly set to $0$ over all elements in
the dataset: obviously such a unit is not informative, and the
corresponding term will vanish both in the metric and in the
gradient.
This is also the case when, e.g., two units incoming to the same unit are
perfectly correlated.  Correlation in the activation profiles
happens systematically in case the size of the dataset is smaller than the
number of incoming parameters at a given unit.

The linear regression viewpoint limits, in theory, the seriousness of
these issues:
this only means the linear regression problem has several
solutions (one can add any quantity to a non-informative weight), and any
of them will do as an update. Indeed, for instance, the matrix $M$ in
Definition~\ref{def:bpmgrad} is of the form $A\transp{A}$,
and $M^{-1}$ is applied to the vector $G$ which is of the form $AY$, thus
$G$ always lies in the image of $A$ and thus the linear system is
underdetermined, not overdetermined.
From the gradient ascent viewpoint this means the
matrix $M$ will be singular but the gradient term $\partial L/\partial w$
will vanish in the corresponding directions.

Numerically, however, the issue must be dealt with. One can use the
Moore--Penrose
pseudoinverse of $M$ (or $F$ in the Fisher matrix case), obtained by adding $\eps.\Id$ to $M$ or to
$F$ with very small $\eps$. This
is a standard regularization technique. It has the advantage of
producing a well-defined update when $\eps\to 0$, asymptotically
independent of $\eps$.

Thus, if a formal definition is needed, one can decide to use the
Moore--Penrose pseudoinverse for $M^{-1}$ and $F^{-1}$ in the definition
of the updates.  However, this formally breaks the invariance properties:
the Moore--Penrose pseudoinverse selects, among the several possible
solutions $(\lambda_i)$ to an underdetermined least squares problem, the
one with smallest norm $\sum \lambda_i^2$, and this is not intrinsic.

\section{A first experimental comparison}
\label{sec:exp}

Although the main focus of this article is theoretical, we performed a
light set of experiments to ensure that the suggested methods are
viable. The companion article \cite{pcnn} contains more in-depth
experiments with recurrent networks and complex symbolic sequences.

To test the influence of the different methods, we chose a very simple
problem in which a perfect solution is expected to be found.  A sparsely
connected network with $5$ layers of size $100$, $30$, $10$, $30$, and
$100$ was built, and $16$ random length-$100$ binary strings were fed to
the input layer, with the target equal to the input (auto-encoding). Ideally the network
learns to encode each of the $16$ samples using $4$ bits on the middle
layer (thus with room to spare) and uses the bottom layer parameters to
rewrite the output from this. This is purely an optimization task without
a learning aspect, as there is no generalization to be done and no
underlying pattern. The focus is on which methods are able to converge to
a good solution.

The sparsely connected network is built at random in each instance as
follows. Each of the $100$ units in the input layer is linked to $5$
randomly selected nodes in the first hidden layer. Each of the $30$
units in the first hidden layer is linked to $5$ random nodes in the
middle hidden layer. The scheme is reversed for the bottom part of the
model: each of the $100$ output units is linked to $5$ random nodes in
the last hidden layer, and each unit in the last hidden layer is linked
to $5$ random nodes of the middle hidden layer.

For each instance, the dataset is made of $16$ random binary strings of
length $100$.  The target for each input is identical to the input. We
use
Bernoulli interpretation of the output.

Note that this setting is adverse for the unitwise and quasi-diagonal
natural gradients, which require a small output layer; this must be
remembered in the comparisons below.

To test the influence of parametrization on non-invariant algorithms, and
to check invariance of the invariant ones, each algorithm was implemented
both using $\sigm(x)$ and $\tanh(x)$ as the activation function.

The methods tested are: backpropagation; unitwise natural gradient;
quasi-diagonal natural gradient; unitwise OP gradient; Monte Carlo
unitwise or quasi-diagonal natural gradient with one sample ($K=1$ in
\eqref{eq:MCnatgrad}); backpropagated metric gradient;
quasi-diagonal backpropagated metric gradient; diagonal Gauss--Newton
(\cite{LBOM96,peskylr}; equivalent to keeping only the diagonal terms in the quasi-diagonal
backpropagated metric gradient); and a batch version of Adagrad/RMSprop
\cite{AdaGrad2011} in which the learning rate for each gradient component
is divided by the root mean square of this component over the samples.

Since the sample size is small, the algorithms were run in batch mode.

\paragraph{Regularization.} 
The algorithms were taken directly from Section~\ref{sec:algos}. To all
methods except backpropagation, we added a
regularization term of $10^{-4}\Id$ to the various matrices involved, to
stabilize numerical inversion.
This value is not so small;
values such as $10^{-7}$ seemed to affect performance. This is probably
due to the small sample size ($16$ samples): each sample contributes a
rank-$1$ matrix to the various metrics. Larger sample sizes would
probably need less regularization.

\paragraph{Initialization.} For the $\tanh$ activation function, all the
weights were initialized to a centered Gaussian random variable of
standard deviation $1/\sqrt{d_j}$ with $d_j$ the
number of units pointing to unit $j$, and the biases set to $0$. For the sigmoid activation
function, the initialization was the corresponding one (using
Eqs.~\ref{eq:sigmtanh1} and~\ref{eq:sigmtanh2}) so that initially the
responses of the networks are the same: namely, each weight was set to a
centered Gaussian random variable of standard deviation $4/\sqrt{d_j}$, and then the bias at unit
$k$ was set to $-\sum_{i\,i\to k}w_{ik}/2$. This initialization has the
property that if the incoming signals to a unit are independent, centered
about $1/2$ (sigmoid) or $0$ (tanh) and of variance $\sigma$ with
$\sigma$ not too large, then the output of the unit is also centered of
variance $\approx \sigma$. (The factor $4$ in the sigmoid case
compensates for the derivative $1/4$ of the sigmoid function at $0$.) See
the argument in \cite{GlorotBengio2010}\footnote{The other initialization
suggested in \cite{GlorotBengio2010}, with weights of magnitude
$\sqrt{6/(d_j+d'_j)}$ with $d'_j$ the number of edges \emph{from} $j$,
did not make any significant difference in our setup.}.

\paragraph{Learning rate.} A simple adaptive method was used for the
learning rate. All methods based on gradients in a metric have a
guarantee of improvement at each step if the learning rate is small
enough. So in the implementation, if a step was found to make the loss
function worse (in a batch mode, thus summed over all samples), the step
was cancelled and the learning rate was divided by $2$. If the step
improves the loss function, the learning rate is increased by a factor
$1.1$. The initial learning rate was set to $0.01$; in practice the
initial value of the learning rate is quickly forgotten and has little
influence.

Unfortunately this scheme only makes sense in batch mode, but it has the
advantage of automatically selecting learning rates that suit each
method, thus placing all methods on an equal footing. 

\paragraph{Execution time and number of iterations.} First, $10,000$ steps of backpropagation
were performed on the whole dataset, in batch mode. The resulting running
time was set aside and converted to an equivalent number of iterations
for all of the other algorithms. This is a very rough way to proceed,
since the running times can depend on the implementation
details\footnote{We tried to implement each method equally carefully.},
and vary from run to run (because floating point operations do
not take the same time depending on the numbers they are operating on,
especially when both very small and very large values are involved).

Most of all, the different methods scale in different ways with the
network, and so the network used here may not be representative of other
situations. In particular this auto-encoder setting with $100$ output
units puts the unitwise natural gradient and quasi-diagonal natural
gradient at a disadvantage (on the same time budget they must performe $n\out$
backpropagations per sample), compared to, e.g., a classification
setting.

Nevertheless, we give in Table~\ref{fig:niter} the numbers of iterations giving roughly equal
running time for each method.

\begin{table}
\begin{center}
\begin{tabular}{lc}
\multicolumn{1}{c}{Method} & Number of iterations
\\\hline
Backpropagation (sigmoid) & 10,000
\\
Backpropagation (tanh) & 10,000
\\
Natural gradient & 9 to 10
\\
Unitwise natural gradient & 2,100 to 2,300
\\
Quasi-diagonal natural gradient & 2,800 to 3,100
\\
Backpropagated metric & 4,200 to 4,300
\\
Quasi-diagonal backpropagated metric & 7,400 to 7,500
\\
Monte Carlo unitwise natural gradient ($K=1$) & 2,800 to 2,900
\\
Monte Carlo quasi-diagonal natural gradient ($K=1$) & 3,900 to 4,000
\\
Unitwise OP gradient & 4,000 to 4,100
\\
Diagonal Gauss--Newton & 7,700 to 7,800
\\
AdaGrad & 8,000 to 8,100
\\
\hline
\end{tabular}
\end{center}
\label{fig:niter}
\caption{Number of iterations resulting in approximately equal execution
times for our problem}
\end{table}

The natural gradient was also tested (using the exact full Fisher matrix
as obtained from Proposition~\ref{prop:fishisfish}). The computational
cost is very high and only $10$ iterations take place in the alloted
time, too few for convergence. Thus we do not report the associated
results.

\paragraph{Results.} In Table~\ref{fig:res}, we report the 
average loss per sample, in bits, and its standard deviation, at the end
of the allocated number of training iterations.  These values
can be interpreted as representing the average number of bits of an
output that the model did not learn to predict correctly (out of $100$). 
The results of the implementations using sigmoid and tanh activation are
reported separately.

Performance as a function of time is plotted in Figure~\ref{fig:fullfig}.

The statistics were made using $20$ independent runs for each method.

\begin{table}
\begin{center}
\begin{tabular}{lcc}
\multicolumn{1}{c}{\multirow{2}{*}{Method}}
& \multicolumn{2}{c}{Average loss (bits) $\pm$ std-dev}
\\
\cline{2-3}
& sigmoid & tanh 
\\
\hline
\emph{\hspace{-.5em}Non-invariant:}
\\
Backpropagation & $35.9\pm2.1$ & $24.7\pm 2.2$
\\
Diagonal Gauss--Newton & $11.6\pm 2.5$ & $3.5\pm 2.0$
\\
AdaGrad &	$51.1 \pm 3.3$ & $25.3\pm 2.0$
\\
\emph{\hspace{-.5em}Invariant:}
\\
Unitwise natural gradient & $0.9\pm 1$ & $1.4\pm 1.8$
\\
Quasi-diagonal natural gradient & $3.5\pm 1.2$ & $3.4\pm 1.6$
\\
Backpropagated metric & $0.8\pm 0.8$ & $0.3\pm 0.5$
\\
Quasi-diagonal backpropagated metric & $1.9\pm1.2$ & $1.5\pm 1.3$
\\
Monte Carlo unitwise natural gradient & $12.9 \pm 1.5$ & $14.1 \pm 2.2$
\\
Monte Carlo quasi-diagonal natural gradient & $7.9 \pm 2.3$ & $10.1 \pm 2.5$
\\
Unitwise OP gradient & $24.7\pm 2.5$ & $28.5 \pm 3.4$
\\\hline
\end{tabular}
\end{center}
\caption{Average loss per sample (bits) after an execution time
equivalent to $10,000$ backpropagation passes, computed over $20$
independent runs, together with standard deviation over the runs}
\label{fig:res}
\end{table}

\begin{figure}
\begin{center}
\includegraphics[width=.98\columnwidth]{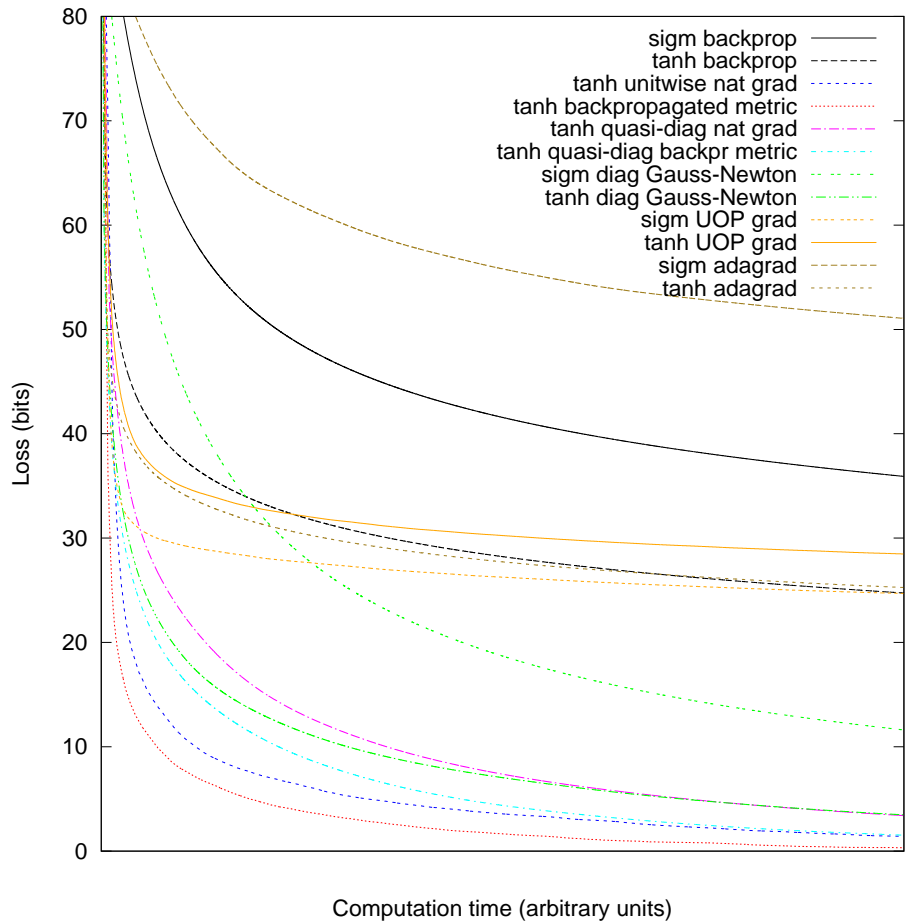}

\caption{\label{fig:fullfig}
Performance over time of all algorithms involved. For better readability
the trajectories of the invariant algorithms have been plotted only in
tanh implementation (Figs.~\ref{fig:mainfig} and~\ref{fig:invtraj} show them in sigmoid
implementation for completeness). Average over 20 runs.}
\end{center}
\end{figure}

\paragraph{Interpretation.} These results are mainly illustrative: the
situation considered here may not be representative because of the small
sample size and network size involved.

Still, it is clear that for problems of this size, the more elaborate
algorithms are very competitive. Only the tanh implementation of the
diagonal Gauss--Newton method comes close to the invariant algorithms
(while its performance in sigmoid implementation is not as good).

As can be expected, the invariant algorithms have similar performance in
sigmoid or tanh implementation: trajectories match each other closely
(Figure~\ref{fig:invtraj}). The variations are caused, first, by random initialization
of the dataset and weights in each run, and second, by the inclusion of
the regularization terms $\eps\Id$, which breaks invariance. If the
effect of the latter is isolated, by having the same initialization in
tanh and sigmoid implementations, the trajectories coincide for the
first few iterations but then start to differ, without affecting overall
performance.

\begin{figure}
\begin{center}
\includegraphics[width=.45\columnwidth]{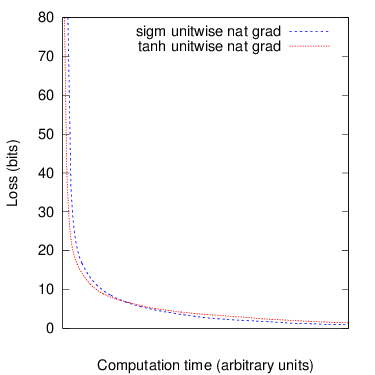}
\includegraphics[width=.45\columnwidth]{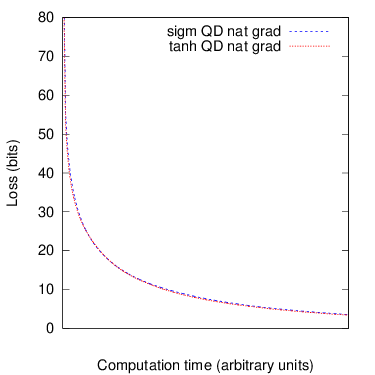}

\includegraphics[width=.45\columnwidth]{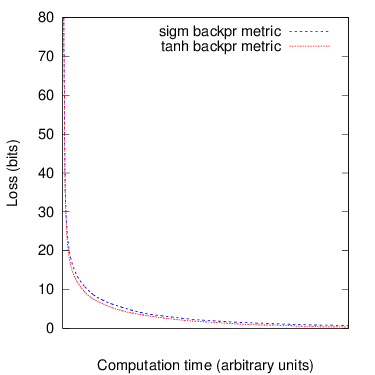}
\includegraphics[width=.45\columnwidth]{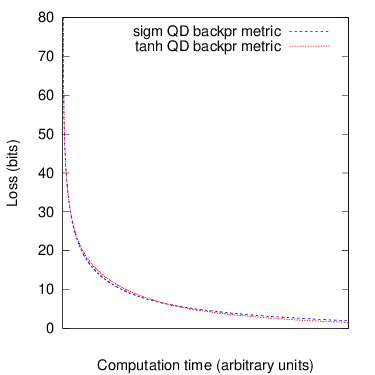}

\caption{\label{fig:invtraj}
Double-checking invariance: Comparison of the trajectories of the invariant algorithms in tanh and
sigmoid implementations}
\end{center}
\end{figure}

In this setting, the natural gradient, in its unitwise and quasi-diagonal
versions, seems to perform slightly worse than the backpropagated metric
methods. This might be an effect of the large output layer size (which
directly affects their computational complexity) combined with a given
computation time budget. Per iteration instead of computation time, the
natural gradient methods perform better than the backpropagated metric
methods, and we expect them to be more competitive for smaller output
sizes, e.g., in classification tasks.

The two algorithms based on squared gradients,
Adagrad and the unitwise OP gradient, both perform rather poorly in this
setting. Adagrad differs from the unitwise OP gradient by using a
diagonal approximation and introducing square roots, which breaks
invariance, while
the unitwise OP gradient is
invariant
and is meant to approximate the natural
gradient. This is a surprise, as, for instance, methods close to the OP
gradient have been found to perform well, e.g., in~\cite{TONGA}, or in~\cite{pcnn} for recurrent networks.
The small size of the dataset in our setting
is not enough to explain this problem, as it does not seem to affect the
other methods. This may be related to the large dimensionality of the
output layer compared
to the number of
samples in our experiment (in contrast to \cite{TONGA} or \cite{pcnn}),
which damages the one-sample OP metric approximation of the natural gradient
and could result in low-rank OP matrices. Indeed, reasoning
on the full (whole-network) metric, the OP gradient contributes a
matrix of rank $1$ for each data sample (see~\eqref{eq:onesamplefisher});
on the other hand, the exact Fisher matrix contributes a sum of $n\out$
matrices of rank $1$ for each data sample as can be seen from
\eqref{eq:fishmod}--\eqref{eq:exactfisherx}. Thus from a theoretical
viewpoint the quality of the one-sample OP approximation of the natural
gradient is likely to depend on output
dimensionality.

Lack of invariance is striking for some algorithms, such as the
diagonal Gauss--Newton method: its performance is very different
in the sigmoid and tanh interpretations (Figure~\ref{fig:diaghess}). The
quasi-diagonal backpropagated metric method only differs from diagonal
Gauss--Newton by the inclusion of a small number of non-diagonal terms in
the update. This change brings the sigmoid and tanh implementations in
line with each other, improving performance with respect to the best of
the two diagonal Gauss--Newton implementations. In settings where the activities
of units (especially internal units, since the input can always be
centered) are not as well centered as here, we expect the quasi-diagonal
backpropagated metric method to outperform the tanh diagonal
Gauss--Newton implementation even more clearly.
Thus the quasi-diagonal backpropagated metric is arguably ``the invariant
way'' to write the diagonal Gauss--Newton algorithm.

\begin{figure}
\begin{center}
\includegraphics[width=.9\columnwidth]{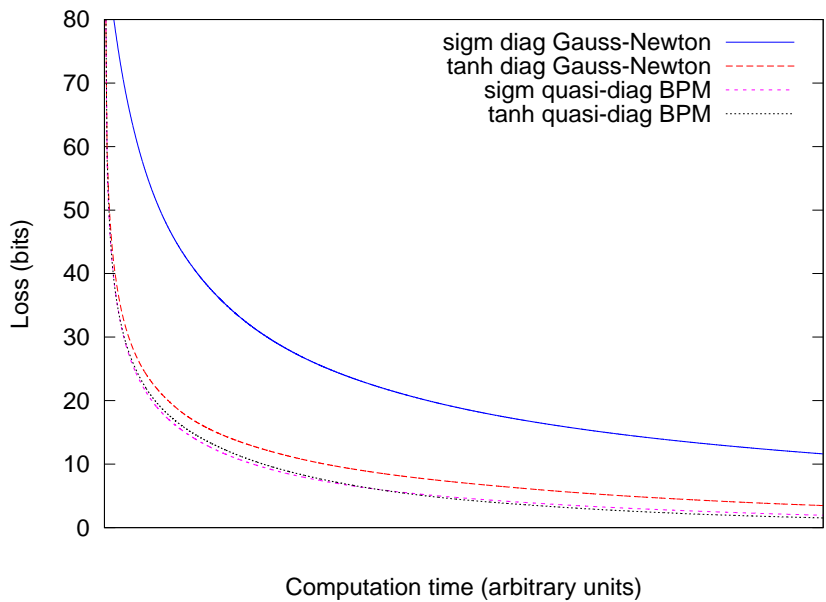}

\caption{\label{fig:diaghess}
Effect of introducing a few non-diagonal terms: Comparison of the
diagonal Gauss--Newton and the quasi-diagonal backpropagated metric
methods}
\end{center}
\end{figure}

In conclusion, although this experiment is a small-scale one, it clearly emphasizes the
interest of using invariant algorithms.

%


\section*{Conclusions}

\begin{itemize}

\item It is possible to write invariant training algorithms for neural
networks that do not have the scalability issues of the natural gradient.
For a network with $n$ units, $n\out$ output units, and at most $d$
incoming connections per unit, we have provided four invariant algorithms for
which the cost of processing a data sample is respectively $O(nd)$,
$O(nd^2)$, $O(ndn\out)$ and $O(nd^2+ndn\out)$. The slower methods are
closer to the natural gradient and have stronger invariance properties.

\item All of these algorithms are mathematically motivated by building
Riemannian metrics on the activity and parameter space of neural
networks, treated as manifolds.

\item The outer product metric encountered in the literature is also
naturally interpreted in this framework. It has a unique property of
spreading improvement most uniformly across the samples at each step.

\item In a small-scale experiment involving an auto-encoding task with
three hidden layers, invariant methods substantially outperform
non-invariant methods. Lack of invariance is clear for some commonly used
methods.

\item The quasi-diagonal backpropagated metric is close to the diagonal
Gauss--Newton method \cite{LBOM96,peskylr} but crucially differs by the inclusion
of a few well-chosen non-diagonal terms. While performance of the
Gauss--Newton method can change substantially from sigmoid to tanh
implementation, the quasi-diagonal backpropagated metric is invariant and
improves performance with respect to diagonal Gauss--Newton in any
implementation.

\end{itemize}

\paragraph{Acknowledgments.} I would like to thank Youhei Akimoto, Jérémy
Bensadon, Cécile Germain, Balázs Kégl, and Michèle Sebag for helpful
conversations and their feedback on this article, as well as the
anonymous referees for careful reading and helpful suggestions.

\appendix

\section*{Appendix}

\renewcommand{\thesubsection}{\Alph{subsection}}

\subsection{Proof of Proposition~\ref{prop:fishisfish}}

The Fisher metric depends, of course, on the interpretation of the output
layer as a probability distribution $\omega$.
The final output $\omega$ is a probability distribution over the values
of $y$ in the target space, parametrized by the activities $a_k$ of the
output units $k$. If the output activities change by $\d a_k$, the 
probability distribution $\omega$ will change as well.
The norm of this change in Fisher metric is
\begin{align}
\natnorm{\d \omega}^2
&=\E_{y\sim\omega} (\delta \ln \omega(y))^2
\\&=
\sum_{k, k'\in\Lout} \E_{y\sim\omega}\frac{\partial \ln \omega(y)}{\partial a_k}
\frac{\partial \ln \omega(y)}{\partial a_{k'}}\d a_k \d a_{k'}
\end{align}
thus stemming from the matrix
\begin{equation}
F_{kk'}\deq\E_{y\sim\omega}\frac{\partial \ln \omega(y)}{\partial
a_k}\frac{\partial \ln \omega(y)}{\partial a_{k'}}
\end{equation}
over the output layer.

In the Bernoulli interpretation, for each component $k$ of the output
layer, the random variable $y_k$ is a Bernoulli variable with parameter
$a_k$.
In the square-loss interpretation, each $y_k$ is a Gaussian random
variable with mean $a_k$ and variance $1$. In the two classification
interpretations, $y$ is a discrete random variable which takes value $k$
with probability $\omega_k=\e^{a_k}/(\sum_{j\in\Lout} \e^{a_j})$ or $\omega_k=a_k^2/(\sum_{j\in\Lout} a_j^2)$.

Let us compute the Fisher metric in the space $\omega$ in each case. In
the Bernoulli case, we have $\omega(y)=\prod_{k\in\Lout}
(\1_{y_k=1}a_k+\1_{y_k=0}(1-a_k))$. Consequently
\begin{equation}
\frac{\partial \ln \omega(y)}{\partial
a_k}=\frac{\1_{y_k=1}}{a_k}-\frac{\1_{y_k=0}}{1-a_k}=
\frac{y_i-a_k}{a_k(1-a_k)}
\end{equation}
Since under the distribution $\omega$ we have $\E y_k=a_k$ and $\Var
y_k=a_k(1-a_k)$ (with $y_k$ and $y_j$ independent for $k\neq j$) we find
\begin{equation}
F_{k k'}=\frac{\1_{k=k'}}{a_k(1-a_k)}
\end{equation}
for $k$ and $k'$ in the output layer.

In the Gaussian case we have $\omega(y)=\prod_k
\frac{\e^{-(y_i-\omega_k)^2/2}}{\sqrt{2\pi}}$ so that $\frac{\partial \ln
\omega(y)}{\partial
\omega_k}=y_i-\omega_k$. Since under the distribution $\omega$ we have
$\E y_k=\omega_k$ and $\Var
y_k=1$ we find $F=\Id$ hence
\begin{equation}
F_{kk'}=\1_{k=k'}
\end{equation}
for $k$ and $k'$ in the output layer.

In the softmax case the probability to have $y=j$ is
$\omega(y)=\e^{a_j}/S$ with $S=\sum_{i\in\Lout} \e^{a_i}$. Thus $\partial
\ln \omega(y)/\partial a_k=\1_{k=j}-\e^{a_k}/S$. Taking the expectation
over $y$ (i.e., over $j$) we find
\begin{align}
F_{kk'}&=
\sum_{j\in\Lout} \frac{\e^{a_{j}}}{S}
\left(\1_{j=k}-\frac{\e^{a_k}}{S}\right)\left(\1_{j=k'}-\frac{\e^{a_{k'}}}{S}\right)
\\&=\frac{\e^{a_k}}{S}\1_{k=k'}
-\frac{\e^{a_k}}{S}\frac{\e^{a_{k'}}}{S}
-\frac{\e^{a_{k'}}}{S}\frac{\e^{a_k}}{S}
+\left(\sum_j \frac{\e^{a_j}}{S}\right)\frac{\e^{a_k}}{S}\frac{\e^{a_{k'}}}{S}
\\&=\frac{\e^{a_k}}{S}\1_{k=k'}-\frac{\e^{a_k}\e^{a_{k'}}}{S^2}
\end{align}

Similarly, in the spherical case the probability to have $y=j$ is
$\omega(y)=a_{j}^2/S$ with $S=\sum_{i\in\Lout} a_i^2$. Thus $\partial
\ln \omega(y)/\partial a_k=2(\frac{\1_{k=j}}{a_j}-\frac{a_k}{S})$. Taking the expectation over $y$ we find
\begin{align}
F_{kk'}&=
4\sum_{j\in\Lout} \frac{a_{j}^2}{S}
\left(\frac{\1_{j=k}}{a_j}-\frac{a_k}{S}\right)\left(\frac{\1_{j=k'}}{a_j}-\frac{a_{k'}}{S}\right)
\\&=4\frac{a_k^2}{S}\frac{1}{a_k}\frac{1}{a_{k'}}\1_{k=k'}
-4\frac{a_k^2}{S}\frac{1}{a_k}\frac{a_{k'}}{S}
-4\frac{a_{k'}^2}{S}\frac{1}{a_{k'}}\frac{a_{k}}{S}
+4\left(\sum_j \frac{a_j^2}{S}\right)\frac{a_k}{S}\frac{a_{k'}}{S}
\\&=\frac4S \1_{k=k'}-\frac{4a_ka_{k'}}{S^2}
\end{align}

These give the expression of the Fisher matrix $F_{kk'}$ for $k$ and $k'$ in the output
layer. This is enough to compute the full Fisher matrix, as follows.

Let $x$ be an
input for the network. Given a variation $\d\theta$ of the network
parameters $\theta$, let $\d a_i$ be the resulting variation of unit $i$,
and let
$\d\omega$ be the resulting variation of the final output $\omega$.
We have
$\d\omega=\sum_{k\in\Lout} \frac{\partial \omega}{\partial a_k}\d a_k$.
The datum-wise Fisher
metric on $\theta$ is
\begin{align}
\natnorm{\d\theta}^2&=\natnorm{\d\omega}^2
\\&=\sum_{k,k'\in\Lout} F_{kk'} \d a_k \d a_{k'}
\end{align}

For each $k$ in the output layer, we have $\d a_k=\sum_i\frac{\partial
a_k}{\partial \theta_i}\d \theta_i$ where the sum runs over all units $i$
in the network. For each $i$ we have $\frac{\partial
a_k}{\partial \theta_i}=\frac{\partial a_k}{\partial a_i}\frac{\partial
a_i}{\partial \theta_i}=J_i^k \frac{\partial
a_i}{\partial \theta_i}$.
Plugging this into the above yields
\begin{equation}
\natnorm{\d\theta}^2=\sum_i\sum_{i'}\sum_{k\in\Lout}\sum_{k'\in\Lout}F_{kk'}J_i^kJ_{i'}^{k'}\frac{\partial
a_i}{\partial \theta_i}\frac{\partial
a_{i'}}{\partial \theta_{i'}}
\end{equation}
so that the term of the Fisher matrix corresponding to $\d\theta_i$ and
$\d\theta_{i'}$ is
$\sum_{k\in\Lout}\sum_{k'\in\Lout}F_{kk'}J_i^kJ_{i'}^{k'}\frac{\partial
a_i}{\partial \theta_i}\frac{\partial
a_{i'}}{\partial \theta_{i'}}
$.

For standard neural networks we have $\delta\theta_i=(\delta
w_{ji})_{j,\,j\to i}$ and moreover
$\frac{\partial
a_i}{\partial w_{ji}}=a_j\deractf_i$.

Plugging into this the expression for $F_{kk'}$ yields the results in
Proposition~\ref{prop:fishisfish}.

%
%
%

\subsection{Proof of Proposition~\ref{prop:equalize}}

Let $v$ be an infinitesimal variation of the parameter $\theta$. Let
$v_i$ be the coordinates of $v$ in some coordinate system. At first
order, the increment in the average loss function along $v$ is $\E_{x\in
\D}\sum_i
\frac{\partial\ell_x}{\partial\theta_i}v_i$.

Let us abbreviate
$\partial_i\ell_x=\frac{\partial\ell_x}{\partial\theta_i}$. The matrix
defining the OP metric is $M_{ij}=\E_{x\in\D} \partial_i\ell_x
\partial_j \ell_x$. The corresponding gradient direction is
$M^{-1}\E_{x\in\D}\partial \ell_x$.

Let $m=\E_{x\in\D}\sum_i \partial_i\ell_x v_i$ be the change in loss
function associated with $v$.
The variance, when $x$ runs over the dataset, of the gain in the loss
function for $x$
is $\E_{x\in\D}\left( (\sum_i \partial_i \ell_x v_i)
-m
\right)^2=\E_{x\in\D}(\sum_i \partial_i \ell_x v_i)^2-m^2$. For fixed
average gain $m$, this is minimal when $\E_{x\in\D}(\sum_i \partial_i
\ell_x v_i)^2$ is minimal.

This is a smooth convex function of $v$, whose minimum we have to find
over the hyperplane $\{v,\, \E_{x\in\D}\sum_i \partial_i\ell_x v_i=m\}$.
The minimum of a positive-definite quadratic functional $\sum_{ij}
A_{ij}v_iv_j$ over a hyperplane $\sum_i B_iv_i=m$, is found at $v=\lambda A^{-1}B$
for some constant $\lambda$. Here we have $B_i=\E_{x\in\D}
\partial_i\ell_x$, and expanding $\E_{x\in\D}(\sum_i \partial_i
\ell_x v_i)^2=\E_{x\in\D}((\sum_i \partial_i
\ell_x v_i)(\sum_j \partial_j
\ell_x v_j))=\sum_{ij}\E_{x\in\D} \partial_i
\ell_x \partial_j \ell_x v_i v_j$ yields $A_{ij}=\E_{x\in\D} \partial_i
\ell_x \partial_j \ell_x=M_{ij}$. Consequently, for any value of $m$, the
optimal $v$ is a multiple of the OP gradient direction
$M^{-1}\E_{x\in\D}\partial
\ell_x$.

\subsection{Definition of the metrics in the formalism of differential
geometry}
\label{sec:formal}

Let $\L$ be the neural network (directed acyclic finite graph of
units); for $k\in \L$, let the activities of unit $k$ belong to a
manifold $\A_k$.
Let the activation function for unit $k$ be
$f_k \from \left(\Theta_k\times \prod_{i\to k} \A_i\right)\to \A_k$,
$(\theta_k,(a_i)_{i\to k})\mapsto f_{\theta_k}((a_i)_{i\to k})$ where
$\Theta_k$ is the manifold of the parameters of unit $k$. Let
$\Lin\subset \L$ and $\Lout\subset \L$ be the input and output layers,
respectively; let $\mathcal{X}$ be the space to which the inputs belong,
and let $\iota\from\mathcal{X}\to \prod_{k\in \Lin} \A_k$ be the input
encoding. Let $\mathcal{O}$ be the manifold to which the outputs belong, and
let $\omega\from \prod_{k\in\Lout} \A_k \to \mathcal{O}$ be the output interpretation.

The values of $a_k$ and of the output $\omega$ can be seen as functions
of the parameter $\theta=(\theta_k)\in\prod_k \Theta_k$ and the input $x$,
by using the induction relations defined by the network:
$a_k(\theta,x)\deq\iota(x)_k$ for $k\in\Lin$,
$a_k(\theta,x)\deq f_k(\theta_k,(a_i(\theta,x))_{i\to k})$ for
$k\not\in\Lin$, and by abuse of notation,
$\omega(\theta,x)\deq \omega((a_k(\theta,x))_{k\in\Lout})$.

Let $\Omega$ be the output metric: a Riemannian metric on the output
manifold $\mathcal{O}$, which to every vector $\d\omega$ tangent to
$\mathcal{O}$ at point $\omega\in\mathcal{O}$, associates its square norm $\Omega_\omega(\d\omega,\d\omega)$ in
a bilinear way. An important example is the Fisher metric when
$\mathcal{O}$ is a space of probability distributions.

Let $\Theta\deq\prod_k \Theta_k$ be the parameter manifold. We are going to
define the natural metric, unitwise natural metric, and backpropagated
metric as Riemannian metrics on $\Theta$.

If $\phi:E\to F$ is a linear map between vector spaces $E$ and $F$, and $g$ is a
bilinear form on $F$, we define the bilinear form $g\circ \phi$ on $E$ by
\begin{equation}
g\circ \phi\from (e,e')\mapsto  g(\phi(e),\phi(e'))
\end{equation}
for any two vectors $e,e'\in E$. If $g$ is positive-semidefinite then so
is $g\circ \phi$.

We denote by $T_pM$ the tangent space to a manifold $M$ at a
point $p\in M$.
Recall \cite[1.36]{GHL87} that if $h\from M \to M'$ is a smooth map between
manifolds, its differential $\frac{\partial h}{\partial p}(p)$ at point
$p\in M$ is a linear map from $T_p M$ to $T_{h(p)}M'$.

Let $\theta\in\Theta$. 
For an input $x$, let $\frac{\partial\omega}{\partial\theta}(\theta,x)$ be the
differential of the network output $\omega(\theta,x)$ with respect to
$\theta$: this is a linear map from $T_\theta\Theta$ to
$T_{\omega(\theta,x)}\mathcal{O}$.

Define the \emph{natural metric}
$g\nat$ as
the bilinear form on $T_\theta\Theta$ for each $\theta$ given by
\begin{equation}
g\nat\deq \frac1{\#\D}\sum_{x\in \D}\, \left(
\Omega_{\omega(\theta,x)}\circ \frac{\partial\omega}{\partial\theta}(\theta,x)
\right)
\end{equation}
where $x$ ranges over
inputs in the dataset
$\D$. By construction, this metric does not depend on any choice of
parametrization and is thus intrinsic.

The \emph{unitwise natural metric} is defined in a likewise manner except that
it first breaks down the tangent vector $\d\theta\in T_\theta\Theta$ into
its components along each unit $k$ using that $\Theta=\prod_k \Theta_k$
and thus $T_\theta\Theta=\bigoplus_k T_{\theta_k}\Theta_k$. The effect is
to make the components $\d\theta_k$ orthogonal. Namely:
\begin{equation}
g\unat(\d\theta,\d\theta)\deq \sum_k
g^k\nat(\d\theta_k,\d\theta_k)
\end{equation}
where $\d\theta=\bigoplus_k \d\theta_k$, and where
\begin{equation}
g^k\nat\deq \frac1{\#\D}\sum_{x\in \D}\left(\Omega_{\omega(\theta,x)}\circ
\frac{\partial\omega}{\partial\theta_k}(\theta,x)\right)
\end{equation}
is the natural metric on $\Theta_k$, with
$\frac{\partial\omega}{\partial\theta_k}(\theta,x)$ the differential
of the network output with respect to $\theta_k$, which is a linear map
from $T_{\theta_k}\Theta_k$ to $T_{\omega(\theta,x)}\mathcal{O}$.

The \emph{backpropagated metric} is defined by backward induction in the
directed acyclic graph $\L$. First, for each input $x$ and each unit
$k$, let us define a
bilinear form $g^{\A_k}\datbp{x}$ on the tangent space $T_{a_k(\theta,x)}\A_k$ to the
activity at $k$. On the output layer let us set
\begin{equation}
g^{\A_k}\datbp{x}\deq \Omega_{\omega(\theta,x)}\circ
\frac{\partial \omega}{\partial a_k}((a_j(\theta,x))_{j\in\Lout})
\qquad\text{for }k\in\Lout
\end{equation}
where $\frac{\partial\omega}{\partial a_k}((a_j(\theta,x))_{j\in\Lout})$ is the
differential of the output interpretation function
$\omega:\prod_{j\in\Lout} \A_j\to\mathcal{O}$ with respect to $a_k$, which is a linear map from
$T_{a_k(\theta,x)}\A_k$ to $T_{\omega(\theta,x)}\mathcal{O}$. Then this is
backpropagated through the network: for each $k$ we define a bilinear
form on $T_{a_k(\theta,x)}\A_k$ by
\begin{equation}
g^{\A_k}\datbp{x}\deq \sum_{i,\,k\to i}
g^{\A_i}\datbp{x}\circ 
\frac{\partial f_i}{\partial a_k}(\theta_i,(a_j(\theta,x))_{j\to i})
\qquad\text{for }k\not\in\Lout
\end{equation}
with $f_i:\Theta_i\times \prod_{j\to i} \A_j\to \A_i$ the activation
function of unit $i$. (If a unit is both an output unit and influences
some other units, we add the two contributions.)
This is transferred to a metric on
$\Theta_k$ via
\begin{equation}
g^{\Theta_k}\datbp{x}\deq g^{\A_k}\datbp{x}\circ
\frac{\partial f_k}{\partial \theta_k}(\theta_k,(a_j(\theta,x))_{j\to k})
\end{equation}
Finally, letting again
$\d\theta=\bigoplus_k \d \theta_k$ be a tangent vector to $\Theta$, define the
backpropagated metric by
\begin{equation}
g\bp(\d\theta,\d\theta)\deq \frac{1}{\#\D}\sum_{x\in\D}\sum_k
g^{\Theta_k}\datbp{x}
(\d\theta_k,\d\theta_k)
\end{equation}
which is a metric on $\Theta$.

Note that these metrics may be non-positive definite (e.g., if a parameter
has no influence on the output).

Since these metrics have been defined using only intrinsic objects
without choosing a parametrization of any of the manifolds $\Theta_k$, they are
intrinsic (for a given output metric $\Omega$). Consequently, when
working in explicit coordinates, the value of the norm of $\d\theta$ is
invariant with respect to any change of variables for each $\theta_k$
(diffeomorphism of $\Theta_k$). The natural metric has the additional
property that it is invariant under changes of variables mixing the
parameters of various units: its invariance group is
$\mathrm{Diff}(\prod_k \Theta_k)$ whereas the invariance group is the
smaller group $\prod_k \mathrm{Diff}(\Theta_k)$ for the unitwise natural
metric and backpropagated metric.

\bibliographystyle{alpha}
\bibliography{gradnn}

\newcommand{\etalchar}[1]{$^{#1}$}
\begin{thebibliography}{OAAH11}

\bibitem[AGS05]{AGS05}
Luigi Ambrosio, Nicola Gigli, and Giuseppe Savar{\'e}.
\newblock {\em Gradient flows in metric spaces and in the space of probability
  measures}.
\newblock Lectures in Mathematics ETH Z\"urich. Birkh\"auser Verlag, Basel,
  2005.

\bibitem[Ama98]{Amari98}
Shun-ichi Amari.
\newblock Natural gradient works efficiently in learning.
\newblock {\em Neural Comput.}, 10:251--276, February 1998.

\bibitem[AN00]{Amari2000book}
Shun-ichi Amari and Hiroshi Nagaoka.
\newblock {\em Methods of information geometry}, volume 191 of {\em
  Translations of Mathematical Monographs}.
\newblock American Mathematical Society, Providence, RI, 2000.
\newblock Translated from the 1993 Japanese original by Daishi Harada.

\bibitem[APF00]{APF00}
Shun{-}ichi Amari, Hyeyoung Park, and Kenji Fukumizu.
\newblock Adaptive method of realizing natural gradient learning for multilayer
  perceptrons.
\newblock {\em Neural Computation}, 12(6):1399--1409, 2000.

\bibitem[Bis06]{Bishop_book}
Christopher~M. Bishop.
\newblock {\em Pattern recognition and machine learning}.
\newblock Springer, 2006.

\bibitem[BL88]{BeckerLeCun88}
Sue Becker and Yann LeCun.
\newblock Improving the convergence of back-propagation learning with second
  order methods.
\newblock Technical Report CRG-TR-88-5, Department of Computer Science,
  University of Toronto, 1988.

\bibitem[DHS11]{AdaGrad2011}
John~C. Duchi, Elad Hazan, and Yoram Singer.
\newblock Adaptive subgradient methods for online learning and stochastic
  optimization.
\newblock {\em Journal of Machine Learning Research}, 12:2121--2159, 2011.

\bibitem[GB10]{GlorotBengio2010}
Xavier Glorot and Yoshua Bengio.
\newblock Understanding the difficulty of training deep feedforward neural
  networks.
\newblock In {\em International Conference on Artificial Intelligence and
  Statistics}, pages 249--256, 2010.

\bibitem[GHL87]{GHL87}
S.~Gallot, D.~Hulin, and J.~Lafontaine.
\newblock {\em Riemannian geometry}.
\newblock Universitext. Springer-Verlag, Berlin, 1987.

\bibitem[Kur94]{Kur94}
Takio Kurita.
\newblock Iterative weighted least squares algorithms for neural networks
  classifiers.
\newblock {\em New Generation Comput.}, 12(4):375--394, 1994.

\bibitem[LBOM96]{LBOM96}
Yann LeCun, L{\'e}on Bottou, Genevieve~B. Orr, and Klaus-Robert M{\"u}ller.
\newblock Efficient backprop.
\newblock In Genevieve~B. Orr and Klaus-Robert M{\"u}ller, editors, {\em Neural
  Networks: Tricks of the Trade}, volume 1524 of {\em Lecture Notes in Computer
  Science}, pages 9--50. Springer, 1996.

\bibitem[LNC{\etalchar{+}}11]{NN_BFGS2011}
Quoc~V. Le, Jiquan Ngiam, Adam Coates, Ahbik Lahiri, Bobby Prochnow, and
  Andrew~Y. Ng.
\newblock On optimization methods for deep learning.
\newblock In Lise Getoor and Tobias Scheffer, editors, {\em ICML}, pages
  265--272. Omnipress, 2011.

\bibitem[OAAH11]{IGO}
Yann Ollivier, Ludovic Arnold, Anne Auger, and Nikolaus Hansen.
\newblock {I}nformation-{G}eometric {O}ptimization algorithms: A unifying
  picture via invariance principles.
\newblock Preprint, arXiv:1106.3708v2, 2011.

\bibitem[Oll13]{pcnn}
Yann Ollivier.
\newblock Riemannian metrics for neural networks {II}: recurrent networks and
  learning symbolic data sequences.
\newblock Preprint, arXiv:1306.0514, 2013.

\bibitem[PB13]{BengioNG2013}
Razvan Pascanu and Yoshua Bengio.
\newblock Revisiting natural gradient for deep networks.
\newblock Preprint, http://arxiv.org/abs/1301.3584, 2013.

\bibitem[RMB07]{TONGA}
Nicolas~Le Roux, Pierre-Antoine Manzagol, and Yoshua Bengio.
\newblock Topmoumoute online natural gradient algorithm.
\newblock In John~C. Platt, Daphne Koller, Yoram Singer, and Sam~T. Roweis,
  editors, {\em NIPS}. Curran Associates, Inc., 2007.

\bibitem[RN03]{RN2}
Stuart Russell and Peter Norvig.
\newblock {\em Artificial Intelligence: A Modern Approach}.
\newblock Prentice-Hall, Englewood Cliffs, NJ, 2nd edition, 2003.

\bibitem[SZL13]{peskylr}
Tom Schaul, Sixin Zhang, and Yann LeCun.
\newblock No more pesky learning rates.
\newblock In {\em Proc. International Conference on Machine learning
  (ICML'13)}, 2013.

\end{thebibliography}

\end{document}